%% file: doc.tex
\documentclass[sigconf]{aamas}

\usepackage{balance} % for balancing columns on the final page

\newcommand{\myAppendix}[1]{#1}

\setcopyright{ifaamas}
\acmConference[AAMAS '22]{Proc.\@ of the 21st International Conference
    on Autonomous Agents and Multiagent Systems (AAMAS 2022)}{May 9--13, 2022}
{Online}{P.~Faliszewski, V.~Mascardi, C.~Pelachaud,
    M.E.~Taylor (eds.)}
\copyrightyear{2022}
\acmYear{2022}
\acmDOI{}
\acmPrice{}
\acmISBN{}

\acmSubmissionID{\# 687}

\input{packages}
\input{commands}

\title[BADDr]{BADDr: Bayes-Adaptive Deep Dropout RL for POMDPs}

\author{Sammie Katt \hfill Hai Nguyen}
\affiliation{
    \institution{Northeastern University}
    \city{Boston}
    \country{USA}}
\email{{katt.s,nguyen.hai1}@northeastern.edu}

\author{Frans A. Oliehoek}
\affiliation{
    \institution{Delft University of Technology}
    \city{Delft}
    \country{Netherlands}}
\email{f.a.oliehoek@tudelft.nl}

\author{Christopher Amato}
\affiliation{
    \institution{Northeastern University}
    \city{Boston}
    \country{USA}}
\email{c.amato@northeastern.edu}

\begin{abstract}
    \input{sections/abstract}
\end{abstract}


\begin{document}

\pagestyle{fancy}
\fancyhead{}

\maketitle

{\fontsize{8pt}{8pt} \selectfont \textbf{ACM Reference Format:} \\ Sammie Katt, Hai Nguyen, Frans A. Oliehoek, and Christopher Amato. 2022. BADDr: Bayes-Adaptive Deep Dropout RL for POMDPs. In {\it Proc. of the 21st International Conference on Autonomous Agents and Multiagent Systems (AAMAS 2022), Online, May 9--13, 2022,} IFAAMAS, 9 pages. }

\section{Introduction: Bayesian RL}
\input{sections/introduction.tex}

\section{Preliminaries}\label{sec:background}
\input{sections/background.tex}

\section{Bayesian Partially Observable RL}\label{sec:brl-formulation}
\input{sections/brl-formulation.tex}

\section{Bayes-Adaptive Deep Dropout RL}\label{sec:baddr}
\input{sections/baddr.tex}

\section{Experiments}\label{sec:experiments}
\input{sections/experiments.tex}

\section{Related Work}\label{sec:related}
\input{sections/related.tex}

\section{Conclusion}
\input{sections/conclusion.tex}

\begin{acks}

    \begin{wrapfigure}{rt}{0.15\columnwidth}
        \vspace{-11pt}
        \hspace*{-.75\columnsep}\includegraphics[width=1.5\linewidth]{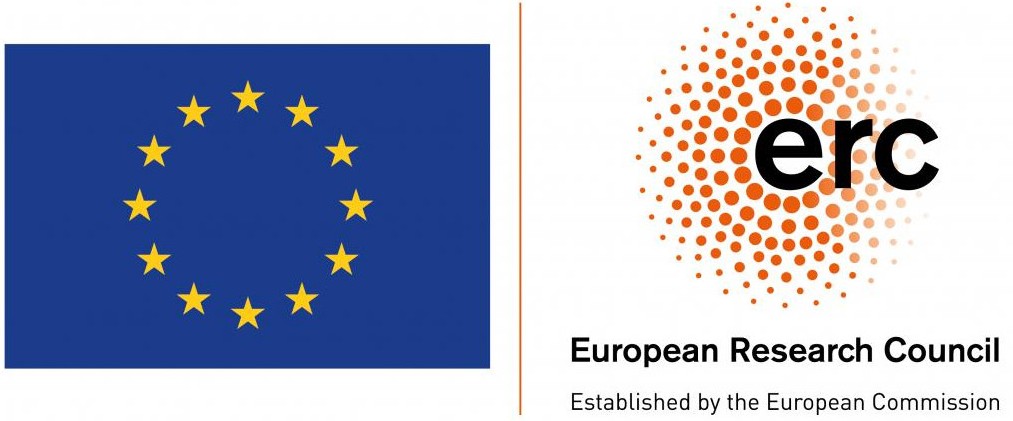}
        \vspace{-11pt}
    \end{wrapfigure}
    This project is funded by NSF grants \#1734497 and \#2024790, Army Research
    Office award W911NF20-1-0265, and European Research Council under the
    European Union’s Horizon 2020 research and innovation programme (grant
    agreement No.~758824 \textemdash INFLUENCE).

\end{acks}

\bibliographystyle{ACM-Reference-Format}
\bibliography{ref.bib}

% XXX: This is my attempt to allow for rendering the paper with and without
% appendix, comment this out _or_ uncomment the `renew` command at the top to
% _not_ render the appendix
\myAppendix{
    \appendix
    \onecolumn
    \input{appendix/outline.tex}

}

\end{document}

%% file: packages.tex
\usepackage{microtype}
\usepackage{graphicx}
\usepackage{booktabs}
\usepackage{wrapfig} 

\usepackage{hyperref}

\usepackage{amsmath,amsfonts,amsmath,amsthm}

\usepackage{cleveref}
\usepackage{cancel}

\usepackage{tikz}
\usetikzlibrary{bayesnet,calc}

\usepackage{algorithm,algorithmic}
\usepackage{subcaption}

\usepackage{footmisc}

%% file: commands.tex
\newtheorem{definition}{Definition}
\newtheorem{lemma}{Lemma}
\newtheorem{thm}{Theorem}

\newcommand{\identityF}{\mathbb{I}}

\newcommand{\natNumbers}{\mathbb{N}} 

\newcommand{\horizon}{K}
\newcommand{\statePrior}{p_{s_0}}
\newcommand{\modelPrior}{p_\mathcal{D}}
\newcommand{\domainSimulatorPrior}{\mathbb{M}}
\newcommand{\rewardFunction}{\mathcal{R}}
\newcommand{\stateSpace}{\mathbb{S}}
\newcommand{\observationSpace}{\mathbb{O}}
\newcommand{\actionSpace}{\mathbb{A}}
\newcommand{\spaceOfDynamics}{\mathfrak{D}}

\newcommand{\att}{_t}
\newcommand{\attplus}{_{t+1}}
\newcommand{\attmin}{_{t-1}}
\newcommand{\seq}[1]{\vec{#1}}

\newcommand{\at}{\seq{a}\attmin}

\newcommand{\ot}{\seq{o}\att}

\newcommand{\st}{\seq{s}\att}
\newcommand{\stplus}{\seq{s}\attplus}
\newcommand{\stmin}{\seq{s}\attmin}

\newcommand{\hist}{h\att}
\newcommand{\histplus}{h\attplus}

\newcommand{\Hist}{H\att}
\newcommand{\Histplus}{H\attplus}

\newcommand{\thetann}{w}
\newcommand{\Thetann}{W}

\newcommand{\ptheta}{\theta}
\newcommand{\pTheta}{\Theta}

\newcommand{\pthetann}{\thetann}
\newcommand{\pThetann}{\Thetann}

\newcommand{\belieft}{p(\mathcal{D}, s_t|\hist)}
\newcommand{\belieftplus}{p(\mathcal{D}, s_{t+1}|\histplus)}

\newcommand{\stateDistrt}{p(\st | \hist)}
\newcommand{\stateDistrtplus}{p(\stplus | \histplus)}

\newcommand{\pparamH}[1]{#1_{\Hist}}
\newcommand{\pparamHplus}[1]{#1_{\Histplus}}

\newcommand{\pthetaH}{\pparamH{\ptheta}}
\newcommand{\pthetaHplus}{\pparamHplus{\ptheta}}

\newcommand{\pthetaT}{\ptheta\att}
\newcommand{\pthetaTplus}{\ptheta\attplus}

\newcommand{\modelDistrt}{p(\mathcal{D}|\Hist)}
\newcommand{\modelDistrtplus}{p(\mathcal{D}|\Histplus)}

\newcommand{\modelDistrthetat}{p(\mathcal{D};\pthetaH)}

\newcommand{\pDynD}{\mathcal{D}(s_{t+1}, o_{t+1}|s_t, a_t)}
\newcommand{\pDynHist}{p(s_{t+1},o_{t+1}|\Hist,a_t)}

\newcommand{\paramUpdate}[1]{\mathcal{U}(#1,s_t,a_t,s_{t+1},o_{t+1})}

\newcommand{\someprior}[1]{p_{#1_0}}
\newcommand{\gbaprior}{\someprior{\bar{s}}}
\newcommand{\pthetaprior}{\ptheta_0}
\newcommand{\pthetannprior}{\pthetann_0}

\newcommand{\algrule}{\par\vskip.2\baselineskip\hrule height .01pt\par\vskip.2\baselineskip}

\newcommand{\algorithmComment}[1]{// \textit{#1}}

\newcommand{\hmdp}{\text{Hist-MDP}} 

\newcommand{\pomdpFamilies}{\mathcal{F}}
\newcommand{\bporl}{\mathcal{M}_{\text{BPORL}}}
\newcommand{\gbapomdp}{\mathcal{M}_{\text{GBA-POMDP}}}
\newcommand{\baddr}{\mathcal{M}_{\text{BADDr}}}

%% file: sections/abstract.tex
While reinforcement learning (RL) has made great advances in scalability,
exploration and partial observability are still active research topics.
In contrast, Bayesian RL (BRL) provides a principled answer to both state
estimation and the exploration-exploitation trade-off, but struggles to scale.
To tackle this challenge, BRL frameworks with various prior assumptions have
been proposed, with varied success.
This work presents a representation-agnostic formulation of BRL under partially
observability, unifying the previous models under one theoretical umbrella. To
demonstrate its practical significance we also propose a novel derivation,
Bayes-Adaptive Deep Dropout rl (BADDr), based on dropout networks. Under this
parameterization, in contrast to previous work, the belief over the state and
dynamics is a more scalable inference problem. We choose actions through
Monte-Carlo tree search and empirically show that our method is competitive
with state-of-the-art BRL methods on small domains while being able to solve
much larger ones.

%% file: sections/introduction.tex
% DRL: impressive
Reinforcement learning~\cite{sutton_introduction_1998} with observable states
has seen impressive advances with the breakthrough of deep
RL~\cite{mnih_playing_2013,silver_mastering_2016,van_hasselt_deep_2016}.
% for POMDPs
These methods have been extended to partially observable
environments~\cite{kaelbling_planning_1998} with recurrent
layers~\cite{hausknecht_deep_2015,wierstra_solving_2007} and much attention has
been paid to encode history into these
models~\cite{karkus_integrating_2018,jonschkowski_differentiable_2018,igl_deep_2018,ma_discriminative_2020}.

% problem DRL: not principled
Although successful for some domains, this progress has largely been driven by
function approximation and fundamental questions are still left unanswered.
The trade-off between between exploiting current knowledge and exploring for
new information, is one such
example~\cite{bellemare_unifying_2016,azizzadenesheli_efficient_2018,osband_randomized_2018}.
Another is how to encode domain knowledge, often abundantly available and
crucial for most real world problems (simulators, experts etc.). Although
research is actively trying to solve these issues, applications of RL to a
broad range of applications is limited without reliable solutions.

% BRL: principled
Interestingly, Bayesian RL (BRL) suffers from opposite challenges. BRL methods
explicitly assume priors over, and maintain uncertainty estimates of, variables
of interest. As a result, BRL is well-equipped to exploit expert knowledge and
can intelligently explore to reduce uncertainty over (only the) important
unknowns.
% problem BRL: scalability
Unfortunately these properties come at a price, and BRL is traditionally known
to struggle scaling to larger problems. For example, the
BA-POMDP~\cite{ross_bayesian_2011,katt_learning_2017} is
a state-of-the-art Bayesian solution for RL in partially observable
environments, but is limited to tabular domains. While factored models can
help~\cite{katt_bayesian_2019}, such representations are not appropriate or may
still suffer from scalability issues.

% contribution
This work combines the principled Bayesian perspective with the scalability of
neural networks. We first generalize previous
work~\cite{ross_bayesian_2011,ghavamzadeh_bayesian_2016,katt_bayesian_2019}
with a Bayesian partial observable RL formulation \emph{without prior
assumptions on parametrization}. In particular, we define the general BA-POMDP
(GBA-POMDP) which, given a (parameterized) prior and update function, converts
the BRL problem into a POMDP with known dynamics. We show that, when the update
function satisfies an intuitive criterion, this conversion is lossless and a
planning solution to the GBA-POMDP results in optimal behavior for the original
learning problem with respect to the prior.
% BADDr
To show its practical significance we derive Bayes-adaptive deep dropout RL
(BADDr) from the GBA-POMDP\@. BADDr utilizes dropout networks as approximate
Bayesian estimates~\cite{gal_dropout_2016}, allowing for an expressive and
scalable approach. Additionally, the prior is straightforward to generate and
requires fewer assumptions than previous BRL methods. The resulting planning
problem is solved with new
MCTS~\cite{browne_survey_2012,silver_monte-carlo_2010} and particle
filtering~\cite{thrun_monte_2000} algorithms.

% evidence
We demonstrate BADDr is not only competitive with state-of-the-art BRL methods
in traditional domains, but solves domains that are infeasible for said
baselines. We also demonstrate the sample efficiency of BRL in a comparison
with the (non-Bayesian) DPFRL~\cite{ma_discriminative_2020} and provide
ablation studies and belief analysis.

%% file: sections/background.tex
\paragraph{Partially observable MDPs}
% POMDP definition
Sequential decision making with hidden state is typically modeled as a
partially observable Markov decision process
(POMDP)~\cite{kaelbling_planning_1998}, which is described by the tuple $(
\mathbb{S, A, O}, \mathcal{D}, \rewardFunction, \gamma, \horizon,
\statePrior)$.
% POMDP spaces
Here $\mathbb{S,A}$ and $\mathbb{O}$ are respectively the discrete state,
action and observation space. $\horizon \in \natNumbers$ is the horizon
(length) of the problem, while $\gamma \in [0,1]$ is the discount factor.
% POMDP dynamics
The dynamics are described by $\mathcal{D}$: $(\mathbb{S} \times \mathbb{A})
\rightarrow \Delta (\mathbb{S} \times \mathbb{O})$, which in practice separates
into a transition and observation model.
% POMDP reward
The reward function $\rewardFunction$: $(\mathbb{S} \times \mathbb{A} \times
\mathbb{S}) \rightarrow \mathbb{R}$ maps transitions to a reward.
% POMDP prior
Lastly, the prior $\statePrior \in \Delta \mathbb{S}$ dictates the distribution
over the initial state.

% life cycle of an agent
At every time step $t$ the agent takes an action $a$ and causes a transition to
a new hidden state $s'$, which results in some observation $o$ and reward $r$.
% goal, history, and belief
We assume the objective is to maximize the discounted accumulated reward
$\sum_{t} \gamma^t r_t$. To do so, the agent considers the observable history
$\hist=(\at,\ot)=(a_{0},o_{1},a_{1},\dots a_{t-1},o_{t})$. Because this grows
indefinitely, it is instead common to use the belief $b \in \mathbb{B}$:
$\Delta \mathbb{S}$, a distribution over the current state and a sufficient
statistic~\cite{kaelbling_planning_1998}. The \emph{belief update}, $\tau$:
$(\mathbb{B} \times \mathbb{A} \times \mathbb{O}) \rightarrow \mathbb{B}$,
gives the new belief after an action and observation, and follows the Bayes'
rule:
\begin{equation}\label{eq:pomdp-belief-update}
    b'(s') =
    \tau(b,a,o)(s') \propto \sum_s \mathcal{D}(s',o|s,a)b(s)
\end{equation}
A policy then maps beliefs to action probabilities $\pi$: $\mathbb{B}
\rightarrow \Delta \mathbb{A}$.

% problem: large POMDPs
In this work we will be concerned with solving large POMDPs in which exact
belief updates and planning are no longer feasible.
% belief solution: particle filters
For efficient belief tracking we use particle filters~\cite{thrun_monte_2000},
which approximate the belief with a collection of `particles' (in this case
states).
% solving large POMDPs: POMCP
For action selection we turn to online planners, since they can spend resources
on only the beliefs that are relevant. In particular, this work builds on
POMCP~\cite{silver_monte-carlo_2010}, an extension of Monte-Carlo tree search
(MCTS)~\cite{kocsis_bandit_2006,browne_survey_2012} to POMDPs, that is
compatible with particle filtering. More details follow
in the method section (\cref{ssec:baddr-solution}).

\paragraph{Dropout neural networks}

% stochastic regularization
Dropout~\cite{li_dropout_2017,gal_dropout_2016,srivastava_dropout_2014,hinton_improving_2012}
is a stochastic regularization technique that samples networks by randomly
dropping nodes (setting their output to zero).
% Bayesian interpretation
A nice property is that it can be interpreted as performing approximate
Bayesian inference. Specifically Gal and Ghahramani~\citep{gal_dropout_2016}
show that applying dropout to a (fully connected) layer $i$ means that its
$K_i$ inputs $j$ are active (or dropped) according to Bernoulli variables
$z_{i,j}$. This means that a (random) effective weight matrix for that layer
randomly drops columns: it can be written as $\tilde{\mathbf{W}}_i =
\mathbf{W}_i \cdot \textrm{diag} ( [z_{i,j}]_{j=1}^{K_i})$. We define $w$ as
the stacking of all such (effective) layer weights, and $\tilde{w} \sim
dropout(\cdot | w)$ as the distribution induced by dropout. Gal and Ghahramani
show that the training objective of a dropout network minimizes the
Kullback-Leibler divergence between $dropout(w)$ and the posterior over weights
of a deep Gaussian process (GP~\cite{Damianou13AISTATS}), a very general
powerful model for maintaining distributions over functions.
% MC estimate
A result of this is a relatively cheap method to compute posterior predictions
using Monte-Carlo estimates:
% MC estimate equation
\begin{equation}\label{eq:dropout-MC-estimate}
    p(y|x) \approx \frac{1}{N} \sum_{n=0}^N p(y|x;\tilde w_n)
\end{equation}

%% file: sections/brl-formulation.tex
% motivation
The strength of Bayesian RL is that it can exploit prior (expert) knowledge in
the form of a probabilistic prior to better direct exploration and thus reduce
sample complexity. However, to operationalize this idea, previous approaches
make limiting assumptions on the form of the prior (such as assuming it is
given as a collection of Dirichlet distributions), which limits their
scalability.

% outline
Here we present the Bayesian perspective of the partially observable RL (PORL)
problem without such assumptions. We first formalize precisely what we mean
with PORL in~\cref{ssec:PORLdef}. \Cref{ssec:brl} then describes the process of
Bayesian belief tracking for PORL in terms of general densities over dynamics.
This makes explicit how the belief can be interpreted as a weighted mixture of
posteriors given the full history (something which we will exploit in
\cref{sec:baddr}). Subsequently, in~\cref{ssec:parameterization} we state a
\emph{parameter update criterion} that provides sufficient conditions for a
parameterized representation to give an exact solution to the original PORL
problem. Finally,~\cref{ssec:brl-formulation:gba-pomdp} then describes how we
can cast the PORL problem as a planning problem using arbitrary parameterized
distributions in the proposed general BA-POMDP (GBA-POMDP).

% contribution
The GBA-POMDP naturally generalizes over previous realizations
(e.g.~\cite{katt_bayesian_2019}), but also support low-dimensional or
hierarchical representations of beliefs. We note that in some cases, such more
compact belief representation have been used in
experiments~\cite{ross_bayesian_2011} even though they were not captured by the
theory presented in the paper. Our paper in that way provides the, thus far
still missing, theoretical underpinning for these experiment. Later
in~\cref{sec:baddr} we will show its practical significance, where we derive a
neural network based realization that is capable of modeling larger problems
than current state-of-the-art Bayesian methods can.

\subsection{Bayesian PORL Definitions}\label{ssec:PORLdef}

Here we formalize the problem of partially observable RL (PORL) and the
Bayesian perspective on it. In PORL the goal is to maximize some metric while
being uncertain about which POMDP we act in:

% definition family of POMDPs
\begin{definition}[Family of POMDPs]\label{def:pomdp-families}
    Given a set of dynamics functions $\spaceOfDynamics$, we say that
    $\pomdpFamilies=\left\{
        (\mathbb{S,A,O},\mathcal{D},\rewardFunction,\gamma,\horizon,\statePrior)\mid\mathcal{D}\in\spaceOfDynamics\right\}
    $
    is a \textbf{family of POMDPs}.
\end{definition}

% setup Bayes PORL
Note that we assume that only the dynamics function is unknown. In our
formulation, the reward function is assumed to be known (even though that can
be generalized, e.g., by absorbing the reward in the state), as well as the
representation of hidden states. We assume that the goal is to maximize the
expected cumulative (discounted) reward over a finite horizon, but other
optimality criteria can be considered.

We now consider Bayesian learning in such families when a prior $\modelPrior$ over
the (otherwise unknown) dynamics is available:

\begin{definition}[BPORL: Bayesian partial observable RL]\label{def:bporl}

    A BPORL model $\bporl=(\pomdpFamilies, \modelPrior)$ is a family of POMDPs
    $\mathcal{F}$ together with a prior over dynamics functions $p_{\mathcal
    D}\in \Delta \spaceOfDynamics$.

\end{definition}

\subsection{Belief Tracking in Bayesian PORL}\label{ssec:brl}

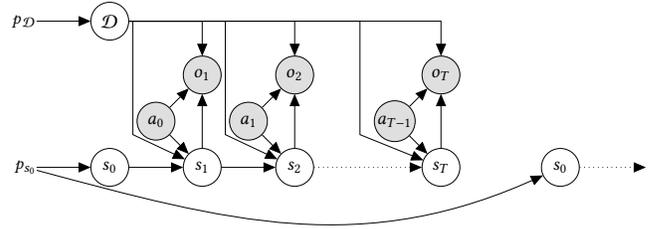
\begin{figure}
    \centering
    \resizebox{\linewidth}{!}{{\input{figures/brl_graph}}}
    \caption{Model of the BRL inference problem. The actions $a$ and
    observations $o$ in gray are observable, which means the policy is
    dependent on them, while the states $s$ and dynamics $\mathcal{D}$ are
    hidden. The priors $p_\mathcal{D}$ and $p_{s_0}$ represent the a-priori
    knowledge. Time is indicated with subscripts and progresses to the right.}
    \label{fig:brl-graph}
    \Description{A graphical representation of the BRL inference problem.
        Describes both POMDP elements such as states, observations and actions, as
        well as the (unknown) dynamics and the Bayesian priors over it.}
\end{figure}

% intro
Here we first derive the equations that describe belief tracking in a BPORL, not
making any assumption on parametrization of these beliefs, but instead assuming
arbitrary densities.
% history notation
The data available to the agent is the observable history $\hist$, the previous
actions and observations, as well as the priors $\modelPrior$ and $\statePrior$
(\cref{fig:brl-graph}), which are implicitly assumed throughout and omitted in
the equations. The quantity of interest is the belief over the current POMDP
state and dynamics $\belieft$.
% regular belief update
We consider how to compute the next belief $\belieftplus$ from a current
$\belieft$ given a new action $a_t$ and observation $o_{t+1}$. Note the
similarities to the POMDP belief update~\cref{eq:pomdp-belief-update} and how
it unrolls over time steps.
\begin{align}
    \belieftplus      & \propto \sum_{s_t} \pDynD \belieft \label{eq:weight-belief-update} \\
    (\text{unroll t}) & \propto \modelPrior(\mathcal{D}) \sum_{\st}
    \statePrior(s_0) \prod_{i=0}^t \mathcal{D}(s_{i+1},o_{i+1}|s_i,a_i)
    \label{eq:unrolled-belief-update}
\end{align}
% % problem: updating weights is insufficient
This assigns more weights to models that are more probable under the evidence
and is fine in general from the Bayesian perspective. Unfortunately, the joint
space of models and states is too large to do exact inference on and, in
practice, we need to resort to approximations and consider only a limited
number of models. The ``true'' model $\mathcal{D}$ will typically not be part
of the tracked models, and merely updating their weights as
in~\cref{eq:unrolled-belief-update} is inadequate: it will lead to degenerate
beliefs where most weights approach zero.
% solution: re-write to update models
To address this, we rewrite~\cref{eq:weight-belief-update} such that it gives a
different perspective, one which updates the models considered by the belief,
and this opens the possibility for combinations with machine learning methods.
% alternative belief formulation
We denote the history including a state sequence $\st$ with $\Hist = (\st,
\hist) = (s_0,a_0,s_1,o_1 \dots a_{t-1},s_t,o_t)$ and apply the chain rule to
formulate the belief as a weighted mixture of model posteriors (one for each
state sequence $\st$):
\begin{equation}\label{eq:mixture-brl}
    \belieftplus = \sum_{\st} \underbrace{\stateDistrtplus}_{\text{weight}}
    \underbrace{\modelDistrtplus}_{\text{component}}
\end{equation}
The advantage is that it includes the term $\modelDistrt$ that can be
interpreted as a posterior over the model given all the data $\Hist$. In the
supplements we show that this belief can be computed recursively:
\begin{gather}
    (\ref{eq:mixture-brl}) % = \sum_{s_t} \sum_{\stmin} \stateDistrtplus
    % \modelDistrtplus = \\
    \propto
    \sum_{\st} \underbrace{\stateDistrt}_{\text{prior weight}}
    \underbrace{\pDynHist}_{\text{transition likelihood}}
    \underbrace{\modelDistrtplus}_{\text{component}} \label{eq:brl-recursive}
\end{gather}

% prior weight term
Here the \textit{prior weight} $\stateDistrt$ is the weight of one of the
components in the belief at the previous time step (\cref{eq:mixture-brl}).
% transition likelihood term
The \textit{transition likelihood} is not trivial and is an expectation over
the dynamics:
% component likelihood
\begin{equation}\label{eq:weight-update}
    \pDynHist = \int_\mathcal{D} p(\mathcal{D}|\Hist) \pDynD
\end{equation}

% component term
Lastly, the \textit{component} $\modelDistrtplus$ is the posterior over the
model given all observable data plus a hypothetical state sequence. This term,
and its computation, is explained in the next section.

\subsection{Parameterized Representations}\label{ssec:parameterization}

% introduction
The last section described the belief in BPORL as a mixture where each
component itself is a distribution over the dynamics (\cref{eq:mixture-brl}).
It also provided the corresponding belief update (\cref{eq:brl-recursive}), but
omitted the computation of the components. Here we show how a posterior
$\modelDistrtplus$ can be derived from a prior component $\modelDistrt$.

% parametrization
In order to make the bridge to practical implementations, we consider the
setting where these distributions are parameterized by $\ptheta \in \pTheta$,
and denote the induced distribution as $\modelDistrthetat$.
% update function
Thus we are now interested in a \emph{parameter update function} that updates
parameters given new transitions: $\mathcal{U}$: $(\pTheta \times \mathbb{S
\times A \times S \times O}) \rightarrow \pTheta$.
% update criterion
This of course raises the question of how such updates can capture the true
evaluation of the posterior $\modelDistrt$. To address this, we formalize a
\emph{parameter update criterion}, which can be used to demonstrate that these
dynamics are sufficiently captured.

% parameter update criterion
\begin{definition}[Parameter update
        criterion]\label{def:parameter-update-criterion}

    We say that the \emph{parameter update criterion} holds \textbf{if it is
        true that}, whenever for some $t$ we have that all
    $\Hist=(s_0,a_0,s_1,o_1,a_1,\dots,a_{t-1},s_t,o_t), a_t$ induce the same
    dynamics as their summary $(\pthetaH,s_t)$, for all $s_{t+1},o_{t+1}$
    \begin{equation}\label{eq:parameter-update-criterion}
        p(s_{t+1}, o_{t+1}|\Hist, a_t) = p(s_{t+1}, o_{t+1}|\pthetaH, s_{t},
        a_t)
    \end{equation}
    \textbf{then}, for the corresponding transitions, and their induced
    $\Histplus = (\Hist,a_t,s_{t+1},o_{t+1})$ and $\pthetaHplus =
        \paramUpdate{\pthetaH}$, the next stage dynamics are also equal, for all
    $s_{t+2},o_{t+2}$:
    \begin{equation*}
        p(s_{t+2},o_{t+2}|\Histplus,s_{t+1},a_{t+1})=p(s_{t+2},o_{t+2}|\pthetaHplus,s_{t+1},a_{t+1}).
    \end{equation*}

\end{definition}

From this, we derive:

% consequence of parameter update criterion

\begin{lemma}\label{lem:parameter-update-criterion}

    If the parameter update criterion holds, and the initial parameter matches
    the prior over models:
    \begin{equation}\label{eq:initial-param-matches-prior}
        \int_{\mathcal D}\modelPrior(\mathcal D)\mathcal D(s_1,o_1|s_0,a_0) =
        p(s_1,o_1|\ptheta_0,s_0,a_0)
    \end{equation}
    then we have that for all $t,\Hist,a_t,s_{t+1},o_{t+1}$
    \begin{equation*}
        p(s_{t+1},o_{t+1}|\Hist,a_t)=p(s_{t+1},o_{t+1}|\pthetaH,s_t,a_t)
    \end{equation*}

\end{lemma}

% interpretation parameter update criterion
Thus, if the parameterization $\ptheta$ can represent the prior over the
dynamics $\modelPrior$ and the parameter update criterion holds, then we can
correctly represent and update the true posterior distribution.

% proof consequence parameter update criterion

\begin{proof}
    The proof follows directly from induction. Base case for $t=0$, in which
    case for $H_0=(s_0)$, holds due to the
    condition~\cref{eq:initial-param-matches-prior}
    \begin{align*}
        p(s_1,o_1|\ptheta_0,s_0,a_0)
        &\stackrel{(\cref{eq:initial-param-matches-prior})}{=} \int_{\mathcal
        D}p(\mathcal D|H_0)\mathcal D(s_1,o_1|s_0,a_0) \\
        &\stackrel{(\cref{eq:weight-update})}{=} p(s_{t+1},o_{t+1}|H_0,a_t) \end{align*}
    At this point we apply the update criterion as our induction hypothesis and
    conclude that the posteriors are identical for all $\Hist$.
\end{proof}

Hence, an update $\mathcal U$ that satisfies the criterion computes
(parameterized) posteriors $\modelDistrtplus$  from a prior component
$\modelDistrt$.

% meaning parameter update criterion
The parameter update criterion captures for instance the updating of statistics
for conjugate distributions, such as Dirichlet-multinomial distributions, but
also situations where the uncertainty about the dynamics functions is captured
by a low-dimensional statistic or where a more general a hierarchical
representation of the dynamics function is appropriate. Approximate inference
methods can also be used to construct parametrizations (of which BADDr will be
one example) and Monte-Carlo simulation can be used if sufficient compute power
is available.

\subsection{General BA-POMDP}\label{ssec:brl-formulation:gba-pomdp}

The previous sections showed how the belief in the BPORL is a mixture of
components, parameterized posteriors over the dynamics, and how to compute
them.
% BRL -> GBA-POMDP
We use this machinery to rewrite the belief update. Specifically, the belief as
a mixture of components (one for each state sequence,~\cref{eq:mixture-brl})
will be represented with weighted state-parameter tuples $(s, \ptheta)$, and
the update (\cref{eq:brl-recursive}) will be reformulated as transitions
between said tuples:
% BRL -> GBA-POMDP equations
\begin{equation*}
    p(\pthetaTplus, s_{t+1}|\histplus) \propto
    \sum_{s_t,\pthetaT}
    \underbrace{p(\pthetaTplus,s_{t+1},o_{t+1}|\pthetaT,s_t,a_t)}_{\text{tuple
    transition probability}} \underbrace{p(\pthetaT, s_t|\hist)}_{\text{prior tuple}}
\end{equation*}
% GBA-POMDP transition
where the transition $p(\pthetaTplus,s_{t+1},o_{t+1}|\pthetaT,s_t,a_t) $
factorizes into
% GBA-POMDP transition equation
\begin{equation*}
    \underbrace{p(\pthetaTplus|\pthetaT,s_t,a_t,s_{t+1},o_{t+1})}_{\text{parameter
    update}}
    \underbrace{p(s_{t+1},o_{t+1}|\pthetaT,s_t,a_t)}_{\text{transition
    likelihood of~\cref{eq:brl-recursive}}}
\end{equation*}
% Parameter update probability
where the parameter update is deterministic and reduces to the indicator
function that returns $1$ \textit{iff} $\pthetaTplus$ equals the result of
$\mathcal{U}$:
% Parameter update probability equation
\begin{equation*}
    p(\pthetaTplus|\pthetaT,s_t,a_t,s_{t+1},o_{t+1}) =
    \identityF(\pthetaTplus,\mathcal{U}(\pthetaT,s_t,a_t,s_{t+1},o_{t+1}))
\end{equation*}
% Intro BA-POMDP
The last mental step interprets the tuples as belief/augmented POMDP states,
and the equations above as POMDP dynamics, which finally leads to the
formulation of the General BA-POMDP\@:

\begin{definition}[General BA-POMDP]\label{def:gbapomdp}

    Given a prior $\pthetaprior$, and a parameter update function $\mathcal U$,
    then the general BA-POMDP is a POMDP\@: $\gbapomdp
        (\pthetaprior, \mathcal U) = ( \mathbb{\bar{S},A,O},\mathcal{\bar{D},
            \bar{R}}, \gamma, \horizon, \gbaprior)$ with augmented state space $\bar{S}
        = (S \times \pTheta)$ and prior $\gbaprior = (\statePrior, \pthetaprior)$.
    $\mathcal{\bar R}$ applies the POMDP reward model $\bar{R}(\bar s, a,
        \bar{s'}) = R(s,a,s)$.  Lastly, the update function $\mathcal U$ determines
    the augmented dynamics model $\mathcal{\bar D}$:
    % Belief POMDP dynamics
    \begin{align}  \label{eq:bapomdp-dynamics}
        \mathcal{\bar{\mathcal{D}}}(\ptheta',s',o|s,\ptheta,a) & \triangleq
        p(\ptheta'|\ptheta,s,a,s',o) p(s',o|\ptheta,s,a)                    \\
                                                               & =
        \identityF(\ptheta',\mathcal{U}(\ptheta,s,a,s',o))p(s',o|\ptheta,s,a)
    \end{align}
\end{definition}

As long as the conditions of~\cref{lem:parameter-update-criterion} hold, the
GBA-POMDPs is a representation of a Bayesian PORL problem. Specifically, any
BPORL and its GBA-POMDP can be losslessly converted to identical `history MDPs'
--- we will call them $\bporl^{\hmdp}$ and $\gbapomdp^{\hmdp}$ --- in which the
states correspond to action-observation histories $h_{t}$.

\begin{thm}\label{thm:gbapomdp-is-bporl}

    Given the POMDP $\gbapomdp = (\pthetaprior,\mathcal{U}) $ of a Bayesian
    PORL problem $\bporl= (\pomdpFamilies,\modelPrior) $ and that the
    parameter update criterion (\cref{def:parameter-update-criterion},
    specifically~\cref{eq:parameter-update-criterion,eq:initial-param-matches-prior})
    hold, then $\gbapomdp^{Hist-MDP}=\bporl^{Hist-MDP}$.

\end{thm}

\begin{proof}
    The basic idea is that we can simply show that due to the matching dynamics
    of~\cref{eq:parameter-update-criterion}, both the rewards $R(h,a)$, as well
    as transition probabilities $T(h'|h,a)$ are identical in the two models.
    Full proof is given in the supplement.
\end{proof}

% merit of GBA-POMDP
The upshot of this is that the GBA-POMDP represents the BPORL problem
\emph{exactly}, meaning that optimal solutions are preserved. In this way, it
facilitates different, potentially more compact, parametrizations of BPORL
problems without necessarily compromising the solution quality. Additionally,
% good properties: like predecessors
like its predecessors, it casts the \textit{learning} problem as a
\textit{planning} problem, opening up the door of the vast body of POMDP
solution methods. This also means that a solution to the GBA-POMDP is a
principled answer to the exploration-exploitation trade-off which leads to
optimal behavior (which respect to the prior). Lastly, because, unlike its
predecessors, it places no assumptions on the prior, it opens the door to a
variety of different machine learning methods, as we will see
in~\cref{sec:baddr}.

\paragraph{Example realization: tabular-Dirichlet}

The BA-POMDP~\cite{ross_bayesian_2011} is the realization of the GBA-POMDP when
choosing the prior parameterization $\pthetaprior$ to be the set of Dirichlets.
The Dirichlet is the conjugate prior to the categorical distribution and comes
with a natural closed-form parameter update: $\mathcal{U}$ in BA-POMDP
increments the parameter (`count') associated with the transition $(s,a,s',o)$.

%% file: figures/brl_graph.tex
\begin{tikzpicture}[every node/.style={scale=1}]

    % states
    \node[latent] (s00) {$s_0$};
    \node[latent,right=of s00] (s10) {$s_1$};

    \node[latent,right=of s10] (s20) {$s_2$};

    \node[latent,right=2cm of s20] (sT0) {$s_T$};

    \node[latent,right=1.5cm of sT0] (s01) {$s_0$};
    \node[right=of s01] (send) {};

    % model
    \node[latent,above=2cm of s00] (D) {$\mathcal{D}$};
    \node[const,left=of D] (pD) {$p_\mathcal{D}$};
    \draw[->] (pD) -- (D);

    % state prior
    \node[const,left=of s00] (ps0) {$p_{s_0}$};
    \edge {ps0} {s00};
    \draw[->] (ps0) to [out=-15, in=205] (s01);

    % observations
    \node[obs,above=of s10] (o10) {$o_1$};
    \node[obs,above=of s20] (o20) {$o_2$};
    \node[obs,above=of sT0] (oT0) {$o_T$};

    % actions
    \node[obs,xshift=-.5cm] (a00) at ($(s10.west)!.5!(o10.west)$) {$a_0$};
    \node[obs,xshift=-.5cm] (a10) at ($(s20.west)!.5!(o20.west)$) {$a_{1}$};
    \node[obs,xshift=-.5cm] (aT0) at ($(sT0.west)!.5!(oT0.west)$) {$a_{T-1}$};

    % state model
    \draw[->] (D.east) -| ($(s00.north)!.5!(a00)$) -- (s10);
    \draw[->] (D.east) -| ($(s10.north)!.5!(a10)$) -- (s20);
    \draw[->] (D.east) -| ($(s10.north)!.5!(a10) + (2.5cm,0)$) -- (sT0);

    \edge {a00} {s10};
    \edge {a10} {s20};
    \edge {aT0} {sT0};

    \edge {s00} {s10};
    \edge {s10} {s20};

    % observation model
    \draw[->] (D) -| (o10);
    \draw[->] (D) -| (o20);
    \draw[->] (D) -| (oT0);

    \edge {s10} {o10};
    \edge {s20} {o20};
    \edge {sT0} {oT0};

    \edge {a00} {o10};
    \edge {a10} {o20};
    \edge {aT0} {oT0};

    % dots
    \draw[dotted,->] (s01) -- (send.east);
    \draw[dotted,->] (s20) -- (sT0);

\end{tikzpicture}

%% file: sections/baddr.tex
% motivation
The GBA-POMDP is a template for deriving effective BPORL algorithms, but it
requires specifying the prior representation and parameter update function.
Here we demonstrate how this perspective can lead to tangible benefits by
deriving BADDr (Bayes-Adaptive Deep Dropout Reinforcement learning), a
GBA-POMDP instantiation based on neural networks. BADDr combines the principled
nature of the GBA-POMDP with the scalability of neural networks and Bayesian
interpretation of dropout. While BADDr introduces some approximations, our
empirical evaluation demonstrates scalability compared to existing BA-POMDP
variants and sample efficiency relative to non-Bayes scalable methods.

% outline
Here we present BADDr as a (GBA-) POMDP, then~\cref{ssec:baddr-solution}
describes the resulting solution method.

\subsection{BADDr: GBA-POMDP using Dropout}

% BADDr as realization and outline
Any GBA-POMDP is defined by its prior and update function. This section defines
BADDr's parameterization, dropout networks $\pthetannprior$, and the parameter
update function $\mathcal U$, which further trains these dropout networks, and
conclude with the formal definition.

\paragraph{BADDr (prior) parameterization}

% dynamics overview
We represent the dynamics prior with a transition and observation model.
The transition model is a neural network parameterized by
$\thetann_{\mathcal{T}}$ that maps states and actions into a distribution over
next states $f_{\thetann_\mathcal{T}}$: $(\mathbb{S} \times \mathbb{A})
    \rightarrow \Delta \mathbb{S}$, and similarly another network
$\thetann_{\mathbb{O}}$ maps actions and next states into a distribution over
observations $f_{\thetann_\mathbb{O}}$: $(\mathbb{S} \times \mathbb{A} \times
    \mathbb{S}) \rightarrow \Delta \mathbb{O}$.
% network in- and output
For each state (observation) feature $n$ we predict the probability of its
values using softmax. In other words, we have an output (logit) $y_{nm}$ for
each value $m$ that feature $n$ can take.
% dropout
Both networks together $\thetann = (\thetann_{\mathcal{T}},
    \thetann_{\mathbb{O}})$ describe the dynamics. As discussed
in~\cref{sec:background}, we interpret dropout as an approximation of a
posterior (recall~\cref{eq:dropout-MC-estimate}):
\begin{equation}\label{eq:baddr-dropout-posterior}
    p(s', o | \pthetann, s, a) \approx \frac{1}{N} \sum_{n=0}^N p(s', o|\tilde w_n, s, a) \text{; }\tilde w_n \sim dropout(\cdot | w)
\end{equation}

\paragraph{BADDr parameter update}

% interpretation and setup
We adopt the perspective of training with dropout as approximate Bayesian
inference (\cref{sec:background}):
\begin{equation}\label{eq:approxBayesian}
    p(s_{t+1}, o_{t+1}|\Hist, a_t) \approx p(s_{t+1},
    o_{t+1}|\pthetann_{\Hist}, s_{t}, a_t).
\end{equation}
This raises the question of what the parameter update function should look
like: assuming $\pthetann_{\Hist}$ captures the posterior over the dynamics
given data $\Hist$, what operation produces the appropriate next weights given
a new transition.
% argue for single-step gradient update
A natural choice is to train the dropout network until convergence on all the
data available ($\Hist$ plus the new transition), however this is
computationally infeasible and unpractical. Instead we argue a reasonable
approximation is to perform a single-step gradient descent step on the new data
point.

% training dnn
We denote a gradient step on parameters $\thetann$ given data point $(s, a, s',
    o)$ as $\nabla \mathcal{L}(\thetann; (s,a), (s',o))$. The loss is defined as
the cross-entropy between the predicted and true next state and observation.
\begin{align}
    \mathcal{L}(\thetann;(s,a),(s',o)) & = -\log p_(s',o|s,a; \thetann)
    \nonumber                                                           \\
    \mathcal{U}(\pthetann,s,a,s',o)    & = \pthetann + \nabla
    \mathcal{L}(\pthetann;(s,a), (s',o)) \label{eq:baddr-parameter-update}
\end{align}

Given the prior and update function, we can now define:

\begin{definition}[BADDr]\label{def:baddr}

    % triplet GBA-POMDP definition
    BADDr is a realization of GBA-POMDP $\baddr = \gbapomdp(\pthetannprior,
        \mathcal U)$ with dropout neural networks $\pthetannprior$ as prior
    parameterization and a single gradient descent step
    (\cref{eq:baddr-parameter-update}) as parameter update $\mathcal U$.
    % dynamics
    The state space of the resulting POMDP is $\bar \stateSpace: (\stateSpace
        \times \pThetann)$, and its dynamics are described as:
    \begin{equation}\label{eq:baddr-dynamics}
        \mathcal{\bar{\mathcal{D}}}(\ptheta',s',o|s,\ptheta,a) \triangleq
        \identityF(\ptheta',\mathcal{U}(\ptheta,s,a,s',o)) p(s', o |
        \pthetann, s, a)
    \end{equation}
    where $p(s', o | \pthetann, s, a)$ is computed according
    to~\cref{eq:baddr-dropout-posterior}

\end{definition}

In this way, BADDR is a specific instantiation of the GBA-POMDP framework. In
BADDr, the parameter update criterion (\cref{eq:parameter-update-criterion})
does \emph{not} hold exactly, since dropout networks only approximate Bayesian
inference. This means that we have to rely on empirical evaluation to assess
the overall performance, which is shown in~\cref{sec:experiments}.

\subsection{Online Planning for BADDr}\label{ssec:baddr-solution}

Here we detail how we use an online planning approach to solve the BADDr model.
We start with the construction of our initial belief, then describe how the
belief is tracked using particle filtering~\cite{thrun_monte_2000}, and finish
with how MCTS is used to select actions.

\paragraph{Constructing the initial prior}

% description and need for a prior
The prior $b(\bar{s}_0) = p_0(s,\pthetann)$ is the product of the prior over
the model and POMDP state, $(p_{\pthetann_0},p_{s_0})$, where $p_{s_0}$ is
given by the original learning problem (of the POMDP).
% weaken standard prior assumption
Weakening the standard assumption in BRL, we do not require a full prior
specification, but assume we can sample domain simulators
$\domainSimulatorPrior$. We believe it is more common to be able to generate
approximate and/or simplified simulators for real world problems than it is to
describe (typically assumed) exhaustive priors.

% ensemble training
The prior specification hidden in $\domainSimulatorPrior$ is translated into a
network ensemble~\cite{dietterich_ensemble_2000} $\{\pthetannprior\}$ by
training each member on a model sampled $\tilde{M} \sim \domainSimulatorPrior$.
The training entails supervised learning on $ ( s,a,s',o ) $ samples with
loss~\cref{eq:baddr-parameter-update}, generated by sampling state-action pairs
uniformly and simulating next-state-observation results (from $\tilde M$).
% Justification prior ensemble
While this leads to an approximation due to the parametric representation,
there is no a problem of data scarcity thanks to the possibility of sampling
infinite data from the prior. Hence when using infinitely large neural
networks, which are universal function approximators, we theoretically could
capture the prior exactly.

% prior: particle filter
Finally the initial particles (belief) is constructed by randomly pairing
states $s_0 \sim \statePrior$ with networks from the ensemble:
$\{(\pthetannprior, s_0)\}^n$.
% subsequent episodes
The prior belief for subsequent episodes is generated by substituting the POMDP
states in the particle filter with initial states sampled from the prior (and
hence maintaining the belief over the dynamics).

\paragraph{Belief tracking}

% introduction RS
Given the initial particle filter and BADDr's
dynamics~\cref{eq:baddr-dynamics}, we use rejection
sampling~\cite{thrun_monte_2000} to track the belief.
% description RS
In rejection sampling (\cref{alg:rs}) the agent samples a particle $(s,\pthetann)$ from
particle filter and simulates the execution of a given action $a$. The
resulting (simulated) new state $(s',\pthetann')$ is added to the new belief
only if the (simulated) observation equals the true observation. Otherwise the
sample is rejected. This process repeats until the new belief contains some
predefined number of particles.

% RS alg
\begin{algorithm}[t]
    \begin{algorithmic}[1]
        \STATE{\textbf{in:} $b$, particle filter $(s,\pthetann)$}
        \STATE{\textbf{in:} $a$, taken action}
        \STATE{\textbf{in:} $o$, new observation}
        \STATE{\textbf{in:} $n$, desired number of particles in next belief}
        \algrule{}
        \STATE{$\bar{b}' \gets \emptyset$ \hfill // next belief, start empty}
        \WHILE{$size(\bar{b}') <  n$}
        \STATE{$(s,\pthetann) \sim \bar{b}$}
        \STATE{// propose sample: BADDr dynamics}\label{alg-line:rs-particle-start}
        \STATE{$\tilde\pthetann \sim \pthetann$ \hfill // dropout sample}\label{alg-line:rs-dropout-sample}
        \STATE{$(s',\tilde{o}) \sim p(\cdot|s,a;\tilde\pthetann)$}
        \STATE{$\pthetann' = \pthetann + \nabla \mathcal{L}(\pthetann, (s_t,a_t), (s_{t+1},o_{t+1}))$}\label{alg-line:rs-particle-end}
        \IF{$\tilde{o} = o$}
        \STATE{Add $(s',\pthetann')$ to $\bar{b}'$ \hfill // correct sampled observation}
        \ENDIF{\hfill // otherwise reject}
        \ENDWHILE{}
        \STATE{\textbf{return} $\bar{b}'$}
    \end{algorithmic}
    \caption{Rejection sampling}\label{alg:rs}
\end{algorithm}

\begin{algorithm}[t]
    \caption{Simulate}\label{alg:po-uct-simulation}
    \begin{algorithmic}[1]
        \STATE{\vphantom{p}\textbf{in:} $s$, POMDP state}
        \STATE{\vphantom{p}\textbf{in:} $\tilde w$ (root-sampled) dynamics; $\tilde w \sim w$}
        \STATE{\vphantom{p}\textbf{in:} $d$, tree depth}
        \STATE{\vphantom{p}\textbf{in:} $h$, action-observation history}
        \algrule{}
        \IF{$terminal(h)$ or $d$ is max depth}
        \STATE{return 0}
        \ENDIF{}
        \STATE{$a \gets ucb(h)$ \hfill \algorithmComment{UCB~\cite{auer_finite-time_2002} using statistics in node $h$}}
        \STATE{$s',o \sim p(\cdot|s,a;\tilde\pthetann)$ \hfill \algorithmComment{use root sampled model as simulator}}\label{alg-line:pouct-transition}
        \STATE{$R \gets \rewardFunction(s,a,s') $ \hfill \algorithmComment{reward function is given}}
        \STATE{$h' \gets (h,a,o)$}
        \IF{$h'  \in tree$}
        \STATE{$r \gets R + \gamma \times simulate((s',\pthetann), d+1, h')$}
        \ELSE{}
        \STATE{$initiate\_statistics\_for\_node(h')$}
        \STATE{$r \gets R + \gamma \times rollout((s',\pthetann), d+1, h')$}
        \ENDIF{}
        \STATE{$N(h,a) \gets N(h,a) + 1$ \hfill \algorithmComment{update statistics}}
        \STATE{$Q(h,a) \gets \frac{N(h,a)-1}{N(h,a)} Q(h,a) + \frac{1}{N(h,a)} r$}
        \STATE{\textbf{return} $r$}
    \end{algorithmic}
\end{algorithm}

\paragraph{Planning}

% intro MCTS
Ultimately we are interested in taking intelligent actions \emph{with respect
    to the belief} --- both over the state and the dynamics. As done in previous
Bayes-adaptive frameworks~\cite{katt_learning_2017,katt_bayesian_2019}, we also
utilize a POMCP~\cite{silver_monte-carlo_2010} inspired algorithm.
% Basics
POMCP builds a look-ahead tree of action-observation futures to evaluate the
expected return of each action. This tree is built incrementally through
simulations (\cref{alg:po-uct-simulation}), which each start by sampling a
state from the belief.
% Contribution: root sampling & dynamics
Our approach is different from regular POMCP in the dynamics being used during
the simulations and is inspired by \emph{model root sampling} in
BA-POMCP~\cite{katt_learning_2017}: When POMCP samples a state $(s,\pthetann)$
from the belief at the start of a simulation, we subsequently sample a model
$\tilde w \sim dropout(\cdot | w)$. This model is then used throughout the
simulation as the dynamics. The computational advantage is two-fold: one, it
avoids computing $\mathcal U$~\cref{eq:baddr-parameter-update} (which requires
back-propagation) at each simulated step. Two, the models in the belief need
not be copied during root sampling because they are never modified (e.g.\ if
simulations were to update $\pthetann$, the pair of networks must be copied to
leave the belief untouched).

%% file: sections/experiments.tex
% intro evaluation
In our experiments we compare BADDr to both a state-of-the-art non-Bayesian
approach and the (factored) BA-POMDP methods. We experiment on smaller
well-known PORL domains, as well as scale up to larger problems.
% observations
Overall our evaluation shows that, one, BADDr is competitive on smaller
problems on which current state-of-the-art BRL methods perform well; and two,
BADDr scales to problems that previous methods cannot. Furthermore, qualitative
analysis show that the agent's belief converges around the correct model in
tiger and that our method outperforms plain re-weighting of models. Lastly, the
strength of Bayesian methods is demonstrated in an comparison with a non-Bayes
model-based representative.

\subsection{Experimental Setup}\label{ssec:experiments:setup}
\input{sections/experiments/setup.tex}

\begin{figure*}[t]
    \begin{subfigure}[b]{.33\linewidth}
        \centering
        \includegraphics[width=1\linewidth]{figures/experiments/tiger/comparison/tiger-performance.pdf}
        \caption{Bayes comparison on tiger.}\label{fig:results-tiger-performance}
    \end{subfigure}
    \begin{subfigure}[b]{.33\linewidth}
        \centering
        \includegraphics[width=1\linewidth]{figures/experiments/road-racer/3lanes/3lane-performance.pdf}
        \caption{Bayes comparison on road race (3 lanes).}\label{fig:results-rr-3}
    \end{subfigure}
    \begin{subfigure}[b]{.33\linewidth}
        \centering
        \includegraphics[width=1\linewidth]{figures/experiments/collision-avoidance/ca_performance.pdf}
        \caption{Bayes comparison on collision avoidance.}\label{fig:ca-experiments}
    \end{subfigure}
    \begin{subfigure}[b]{.33\linewidth}
        \centering
        \includegraphics[width=1\linewidth]{figures/experiments/road-racer/9lanes/9lane-performance.pdf}
        \caption{Our work on road race (9 lanes).}\label{fig:results-rr-9}
    \end{subfigure}
    \begin{subfigure}[b]{.33\linewidth}
        \centering
        \includegraphics[width=1\linewidth]{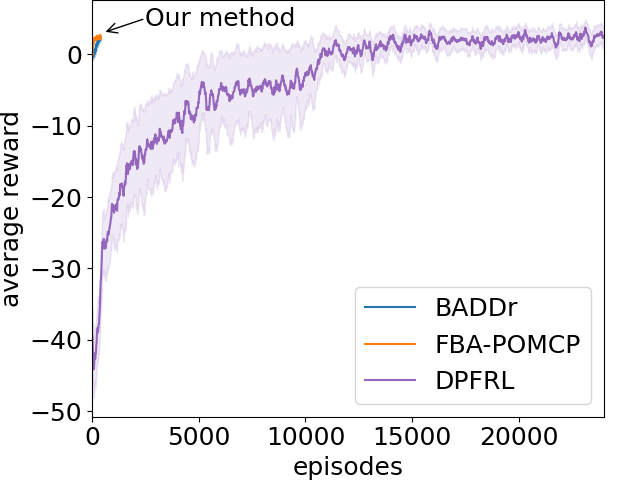}
        \caption{Non-Bayes comparison on tiger.}\label{fig:dpfrl-comparison-tiger}
    \end{subfigure}
    \begin{subfigure}[b]{.33\linewidth}
        \centering
        \includegraphics[width=1\linewidth]{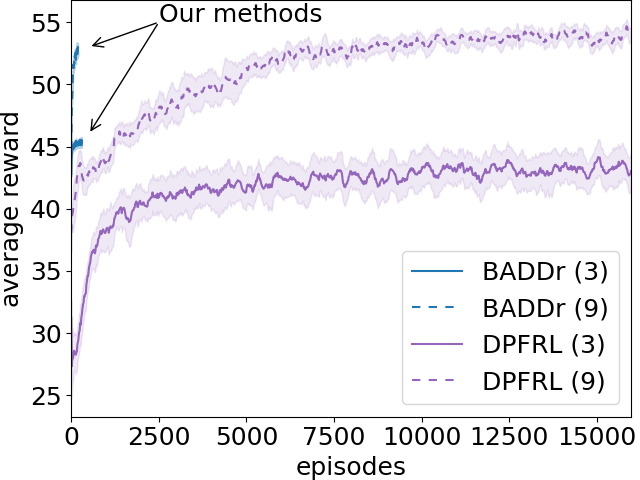}
        \caption{Non-Bayes comparison on road race (3 \& 9).}\label{fig:dpfrl-comparison-rr}
    \end{subfigure}
    \begin{subfigure}[b]{.33\linewidth}
        \centering
        \includegraphics[width=1\linewidth]{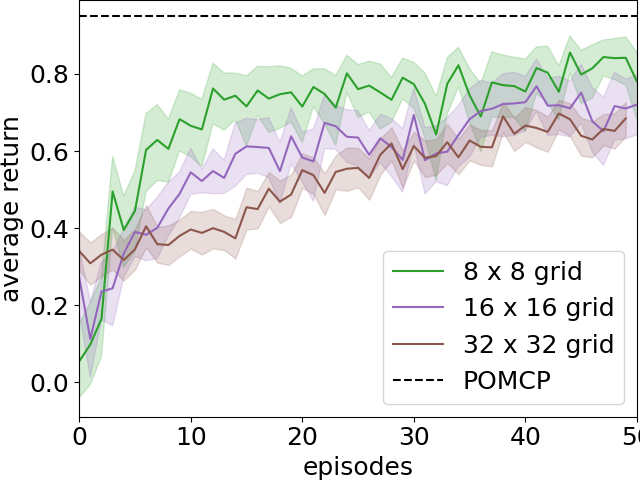}
        \caption{Our work on gridverse.}\label{fig:results-gridverse}
    \end{subfigure}
    \begin{subfigure}[b]{.33\linewidth}
        \centering
        \includegraphics[width=1\linewidth]{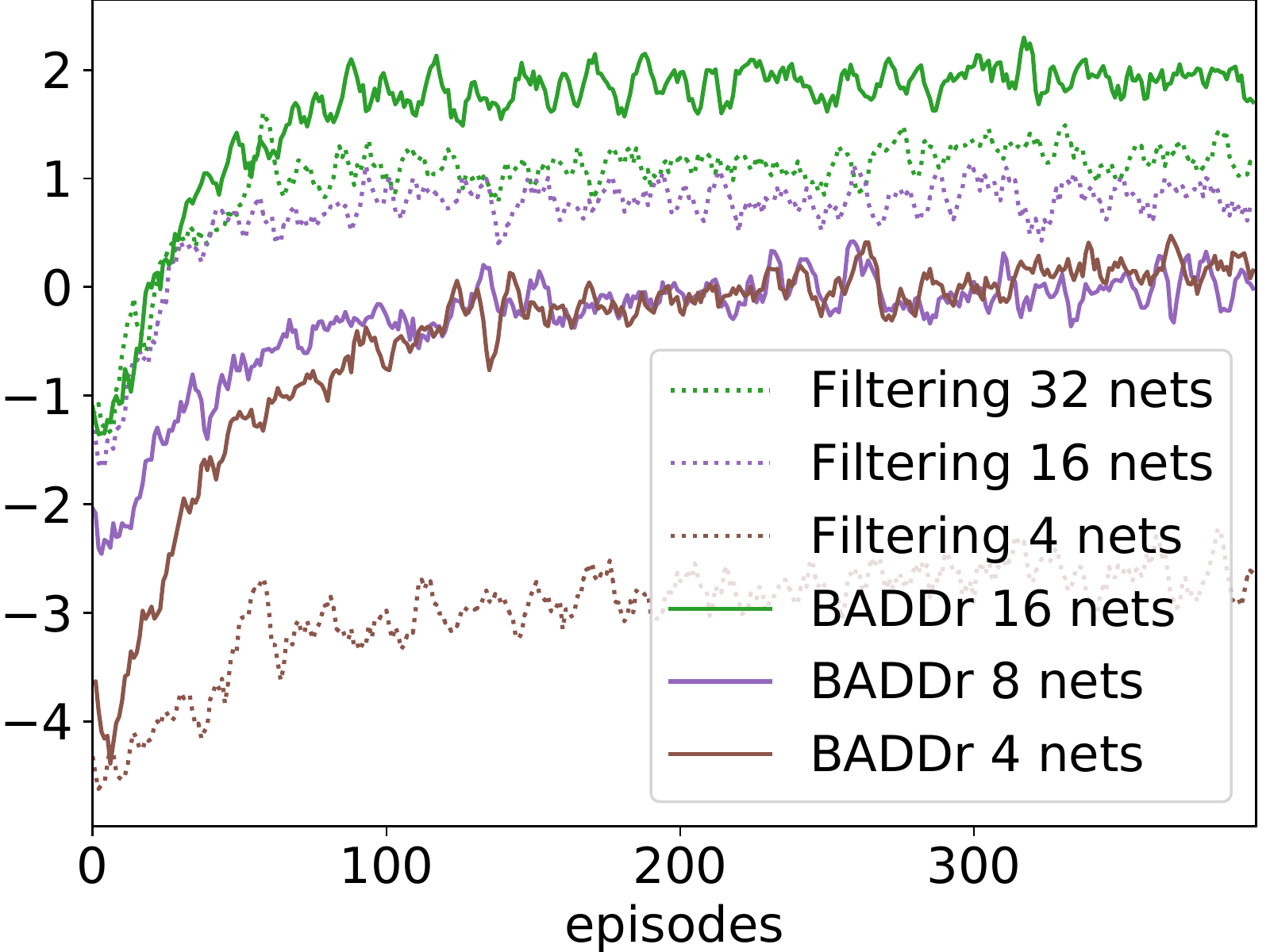}
        \caption{Ablation (no model updates) on tiger.}\label{fig:ensemble-comparison}
    \end{subfigure}
    \begin{subfigure}[b]{.33\linewidth}
        \centering
        \includegraphics[width=1\linewidth]{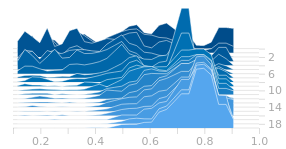}
        \caption{Belief of tiger observation model.}\label{fig:belief-tiger-hist}
    \end{subfigure}
    \caption{%
        % (Bayes) performance
        Our work (blue) is competitive with FBA-POMCP in small problems (a \&
        b), and can scale to larger instances (d \& g). Fig. (c) shows that
        BADDr struggles when prior certainty is crucial.
        % Non-Bayes comparison bottom row left figure: ablation
        Fig (e \& f) compares with DPFRL, where BADDr shows both a better
        initial performance due to exploiting the prior and better sample
        efficiency. Dotted lines represent upper bound by running POMCP on the
        true POMDP\@.
        % Analysis
        Fig. (h) demonstrates BADDr (solid) requires far fewer models than an
        ablation method that only re-weights models in its beliefs (dotted).
        Fig. (i) shows the belief in BADDr on tiger converges to the true
        value (0.85).
    }
    \Description{%
        % (Bayes) performance
        Our work (blue) is competitive with FBA-POMCP in small problems (a \&
        b), and can tackle larger instances (d \& g). Fig. (c) shows that BADDr
        struggles when prior certainty is crucial, but outperforms the baseline
        otherwise.
        % Non-Bayes comparison bottom row left figure: ablation
        Fig (e \& f) compare our methods with DPFRL\@. Our methods show both a
        better initial performance due to exploiting the prior and better
        sample efficiency. Dotted lines represent upper bound by running POMCP
        on the true POMDP\@.
        % Analysis
        Fig. (h) compares BADDr (solid) with an ablation method that only
        re-weights models in its beliefs (dotted). This ablated approach
        requires far more initial models which does not scale problems with
        more variables. Fig. (i) shows the model posterior of BADDr in tiger
        converges to true value (0.85).
    }
\end{figure*}

\subsection{Bayesian RL Comparison}\label{ssec:experiments:bayes}
\input{sections/experiments/brl-results.tex}

\subsection{Non-Bayesian RL Comparison}\label{ssec:experiments:non-bayes}
\input{sections/experiments/dpfrl-results.tex}

%% file: sections/experiments/setup.tex
\paragraph{Baselines}

% Bayes
We compare with (F)BA-POMCP~\cite{katt_bayesian_2019} as the state-of-the-art
BRL baseline. To ensure a fair comparison we use their prior as the generative
process $\domainSimulatorPrior$ to sample POMDPs from when constructing our
prior. Additionally, for all domains, the parameters shared among FBA-POMCP and
BADDr (number of simulations \& particles, the UCB constant, etc.) are the
same. We also plot ``POMCP''' on the true models as an upper bound as dotted
lines.

% DPFRL
We also include discriminative particle filter reinforcement learning
(DPFRL~\cite{ma_discriminative_2020}) as a baseline in our experiments. DPFRL
is a novel end-to-end deep RL architecture designed specifically for partial
observable environments. We used the official implementation and fine-tuned by
picking the best performing combination of the number of particles, learning
rate and network sizes.

\paragraph{Small domains \& their priors}

% Tiger domain
The experiments on the tiger problem~\cite{kaelbling_planning_1998} function as
a baseline comparison. This problem is well known for being tiny but otherwise
highly stochastic and partial observable. The prior here is a single dropout
network trained on the expected model of the prior used in (F)BA-POMCP\@.

% collision avoidance
In collision avoidance~\cite{luo_importance_2019}, the largest problem solved
with FBA-POMCP~\cite{katt_bayesian_2019}, the agent is a plane flying from the
right column of a grid to the left. The last column is occupied by a moving
single-cell obstacle that is partially observable and must be avoided. This
task is challenging in that \emph{both} the observation and transition model
are highly stochastic.
% ca: prior
Again we employ their prior --- uncertainty over the behavior of the obstacle
--- to train our ensemble.

% road racing
We designed the road racing problem, a variable-sized POMDP grid model of
highway traffic, in which the agent moves between three lanes in an attempt to
overtake other cars (one in each lane). The state is described by the distance
of each of those cars in their respective lanes and the current occupied lane.
During a step the distance of the other cars decrements with some probability.
The speed, and thus the probability of a car coming closer, depends on the
lane.  The initial distance of all cars is 6, and when their position drops to
-1, as the agent overtakes them, it resets. The observation is the distance of
the car in the agent's current lane, which also serves as the reward,
penalizing the agent for closing in on cars.
% rr: prior description
The prior over the observation model and the agent's location transitions is
correct. The speed that is associated with each lane, however, is unknown. A
reasonable prior is to assume no difference, so we set the expected
probability of advancing to 0.5 for all lanes.

\paragraph{Large domains \& their priors}

% road racing
We run an additional larger experiment of the road racing problem with nine
lanes. This significantly increases the size of the problem and, as will be
mentioned in the results, makes previous frameworks intractable.

% gridverse description
The last and largest domain is gridverse. Here the agent must navigate from one
corner of a grid to the goal in the other, while observing only the cells in
front (a beam of width 3 leading to up to $96$ observation features). We run
this on a grid of up to 32 by 32 cells.
% prior
In this environment we assume the observation model is given and learn the
transition model of the agent's position and orientation. For our prior we
learn on data generated by a simulator with the correct dynamics for
``rotations'', but a noisy ``forward'' action. The challenge for the agent is to
correctly infer the distance of this action (and thus its own location) online.

%% file: sections/experiments/brl-results.tex
\paragraph{Small domains}

% expectations small domains
The smaller domains are generally compactly modeled by the (F) BA-POMDP\@. As a
result the baseline BRL methods are near optimal and we cannot expect to do
much better. Rather these experiments test whether BADDr is sample efficient
even when compared to optimal representations.

% tiger results
\Cref{fig:results-tiger-performance} compares our method with (factored)
BA-POMCP on tiger. Unsurprisingly, the tabular representation has a slight
advantage initially thanks to the sample efficiency. After twice the amount of
data, our method catches up and reaches the same performance. Although Tiger is
widely considered a toy problem due to its size, the inference and resulting
planning problem are hard: a slight difference in the belief over the model
significantly alters the optimal policy. BADDr's performance here showcases the
ability to tackle highly stochastic and partially observable tasks.

% belief
We also investigate how well BADDr captures the posterior over the model.
\Cref{fig:belief-tiger-hist} shows the belief over the probability of hearing
the tiger behind the correct door in a particular run. Initially the prior is
uncertain and its expectation is incorrect, but over 20 episodes the belief
converges to the true value of 0.85.

% CA: performance
\Cref{fig:ca-experiments} shows BADDr performs nearly as well as FBA-POMCP on
the collision avoidance problem. We hypothesize that the representational power
of the Dirichlets, in contrast with an ensemble of dropout networks, explains
the discrepancy.
% explanation discrepancy
Specifically, the Dirichlet allow more control over the \emph{certainty} of the
prior: the agent prior over the observation model is confident (high number of
counts), and is uncertain over the transition of the obstacle. Admittedly, such
a prior is difficult to capture in an ensemble (and thus in BADDr). We test
this hypothesis by running FBA-POMCP with an equally uncertain prior, called
`FBA-POMCP\@: uncertain prior'. Results show that BADDr performs somewhat in
the middle of both. Hence in some occasions the prior representation of BADDr
results in diminished performance, but BADDr shows higher potential when both
methods are provided similar prior knowledge (as BADDr outperforms `FBA-POMCP:
uncertain prior'). Note that FBA-POMCP is given the correct (sparse) graphical
model, which is a strong assumption in practice that simplifies the learning
task.

% discussion: 3 lanes
On the 3-lanes road racing domain (\cref{fig:results-rr-3}) the difference
between our work and BA-POMCP is nonexistent. This again confirms that BADDr is
competitive with state-of-the-art BRL methods on small problems which these
methods are designed for and perform near optimal in. Unlike real applications,
these problems are compactly represented by tables and improvements are
unlikely.

\paragraph{Larger domains}

% discussion: 9 lane
The advantage of our method becomes obvious in larger problems. In the 9 lanes
problem (\cref{fig:results-rr-9}), for instance, even FBA-POMCP has $10^{13}$
entries, is unable represent this compactly, and runs out of memory. But a
dropout network of 512 nodes can model the dynamics well enough: despite the
increasing size of the problem, the learning curve is similar to the smaller
problem (\cref{fig:results-rr-3}) and a similar amount of data is needed to
nearly reach the performance of POMCP in the POMDP\@. This suggests that there
is some pattern or generalization that BADDr is exploiting.

% results
\Cref{fig:results-gridverse} shows the performance on gridverse on a grid of
size 8, 16 and 32. The agent learns to perform nearly as well as if given the
true model, indicated by the dotted line (for all 3 sizes), with only a small
visible effect of significantly increasing the size of the problem. Note again
that for this domain tabular representations are infeasible: a single particle
would need to specify up to $10^7$ parameters (roughly 40GB of memory for a
64bit system). Bayes networks (FBA-POMDP) is unable to exploit the structure of
this domain, which does not show itself through independence between features.
However, the dropout networks in the belief of BADDr can generalize to problems
otherwise too large to represent.

\paragraph{Ablation study: re-weighting}

% reminder plain re-weighting
\Cref{ssec:brl} claimed that, while technically correct, solely re-weighting
models in the belief (\cref{eq:weight-belief-update}) leads to belief
degeneracy as it relies on a correct (or good) model to be present in the
initial belief.
% description experiments
This experiments verifies that claim by assessing the performance of plain
re-weighting, denoted with `filtering'. This is implemented by omitting the
parameter update function $\mathcal U$ (e.g.\ SGD step).
% vanilla filtering results
\Cref{fig:ensemble-comparison} shows that the performance of both filtering
(dotted) and BADDr (solid) increases as a function of the number of models in
the prior. However, BADDr is significantly more efficient: `filtering' 128
models in the tiger problem performs similarly to BADDr with just 8. Hence,
while theoretically possible, it requires too many models to be useful.

\paragraph{Run-time}

BADDr is a little slower than FBA-POMCP, since calling (i.e.\ planning) and
training neural networks (i.e\ belief update) is computationally more expensive
than tables. In practice, however, we found this insignificant. In environments
where tabular approaches are applicable and fit in memory the difference was at
most a small factor. As a result, a single seed/run of any experiment finished
within a day on a typical CPU-based Architecture.

%% file: sections/experiments/dpfrl-results.tex
% introduction
We also ran DPFRL on all domains to investigate the differences between
Bayesian and non-Bayesian approaches. Results on tiger
(\cref{fig:dpfrl-comparison-tiger}) and both road race instances
(\cref{fig:dpfrl-comparison-rr}) have been picked out as representative, but
other results looked similar and have been included in the appendix. Note the
performance of the Bayesian methods are identical to previous plots; only the
x-axis is different.

% observation: sample efficiency
While the eventual performance is similar, the difference in learning speed is
immediately obvious. Where the BRL methods learn within tens to hundreds of
episodes, DPFRL requires up to tens of thousands --- a direct consequence of
the sample-efficiency and exploration provided by the Bayesian perspective. In
general we found when measuring the number of episodes necessary to reach
similar performance, that BADDr was at least 40x  (and up to 1200x) more sample
efficient than DPFRL\@.

% observation: prior exploitation
Another advantage of BRL is the exploitation of a prior, which is visualized by
the discrepancy in the \emph{initial performance}. In the tiger domain DPFRL
starts from scratch with a return of ${-}40$ by randomly opening doors. In real
applications with real consequences, this can be a huge problem and the ability
to encode domain knowledge is crucial. Random behavior is less problematic in
the road race domain, yet also there it takes DPFRL thousands of episodes (of
many time steps) to reach the performance BADDr has \emph{at the start}.

%% file: sections/related.tex
% BAPOMDP
Bayesian RL for discrete POMDPs typically adopts the Dirichlet prior approach
taken in the BA-POMDP~\cite{ross_bayesian_2011,ghavamzadeh_bayesian_2016}. For
instance, before generalized to unknown structures with the
FBA-POMDP~\cite{katt_bayesian_2019}, prior work represented the model posterior
as Dirichlets over Bayes network parameters~\cite{poupart_model-based_2008}.
% MEDUSA
Other work circumvents the need for mixtures to represent the posterior by
assuming access to an oracle to provide access to the underlying
state~\cite{jaulmes_learning_2005,jaulmes_formal_2007}. By exploiting this
information, they approximate the belief with a MAP estimate of the counts.
% iPOMDP
A notable exception to this line of work is the
iPOMDP~\cite{doshi-velez_infinite_2009}. This work is more general in that
knowledge of the state \emph{space} is not assumed a-priori. Dropping this
assumption means that it is impossible to sum over state sequences, and hence
our formulation is not compatible.
% Continuous
Bayesian methods for continuous POMDPs generally assume Gaussian dynamics and
model the belief with a GP\@. The methods in the literature vary in their
assumptions, such as restricting to a MAP
estimate~\cite{dallaire_bayesian_2009}, simplifying the observation model to
Gaussian noise around the state~\cite{mcallister_data-efficient_2017}, full
access to the state during learning~\cite{deisenroth_solving_2012}. More
similar in spirit to our method is the
BA-Continuous-POMDP~\cite{ross_bayesian_2008}, as it maintains a mixture of
model posteriors, Normal-inverse-Wishart parameters, and presents a similar
derivation to ours (specific to said parameterization).

% model-free BRL
In contrast, Bayesian \emph{model-free} approaches maintain a distribution over
the policy with, for example, ensembles~\cite{osband_randomized_2018} or
Bayesian Q-networks~\cite{azizzadenesheli_efficient_2018}. While proven
successful in their exploration-dependent domains, they have not been tested
under partial observability.

% Model-based MDP
Model-based RL for fully observable MDPs~\cite{chua_deep_2018} is better
understood and include ``world models''~\cite{ha2018world,hafner2020mastering}
and Dyna-Q based methods~\cite{yu2020mopo,janner2019trust}. Bayesian
counterparts include ensemble methods~\cite{rajeswaran_epopt:_2017},
variational Bayes (variBad~\cite{zintgraf_varibad_2019}), and GP-based
models~\cite{deisenroth2011pilco} (similar to us extended with a dropout
network approximation of the dynamics~\cite{gal2016improving}. Most relevant
here is the work on the Bayes-adaptive
MDP~\cite{wyatt_exploration_2001,guez_sample_based_2015,duff_optimal_2002,8202242}.
MDPs, however, are strictly easier and these methods do not trivially extend to
partial observability.

%% file: sections/conclusion.tex
% recap BRL
Bayesian RL for POMDPs provide an elegant and principled solution to key
challenges of exploration, hidden state and unknown dynamics. While powerful,
their scalability and thus applicability is often lacking.
% contribution
This paper presents a rigorous formulation of the General Bayes-adaptive POMDP,
as well as a novel instantiation, BADDr, which improves scalability while
maintaining sample efficient with dropout networks as a Bayesian estimate
of the dynamics.
% sales pitch
The empirical evaluation shows our method performs competitively with
state-of-the-art BRL on small problems, and solves problems that were
previously out of reach. It also demonstrates the strengths of Bayesian
methods, the ability to encode prior and guide exploration, through a
comparison with the non-Bayesian DPFRL\@.

%% file: appendix/outline.tex
\section{Proofs}\label{sec:proofs}

The paper defers two proofs: the derivation of equation (6) on page 3 and the
theorem (1) on page 4. Here in~\cref{ssec:proofs:brl-derivation} the equation
is derived, and~\cref{ssec:proofs:equivalence-lemma} proofs the lemma.

\subsection{BRL Derivation}\label{ssec:proofs:brl-derivation}
\input{appendix/brl-derivation.tex}

\subsection{Equivalence of History MDPs}\label{ssec:proofs:equivalence-lemma}
\input{appendix/hist-equivalence.tex}

\section{Road Racing Domain}
\input{appendix/lanes.tex}

\section{(Hyper) Parameters}
\input{appendix/parameters.tex}

\section{DPFRL Details}
\input{appendix/dpfrl.tex}

%% file: appendix/brl-derivation.tex
Here we the derive the claim in that the belief can be computed by:

\begin{equation}
    \sum_{\st} \stateDistrtplus \modelDistrtplus = \eta \sum_{\st} \stateDistrt \pDynHist
    \modelDistrtplus \label{app-eq:brl-recursive}
\end{equation}

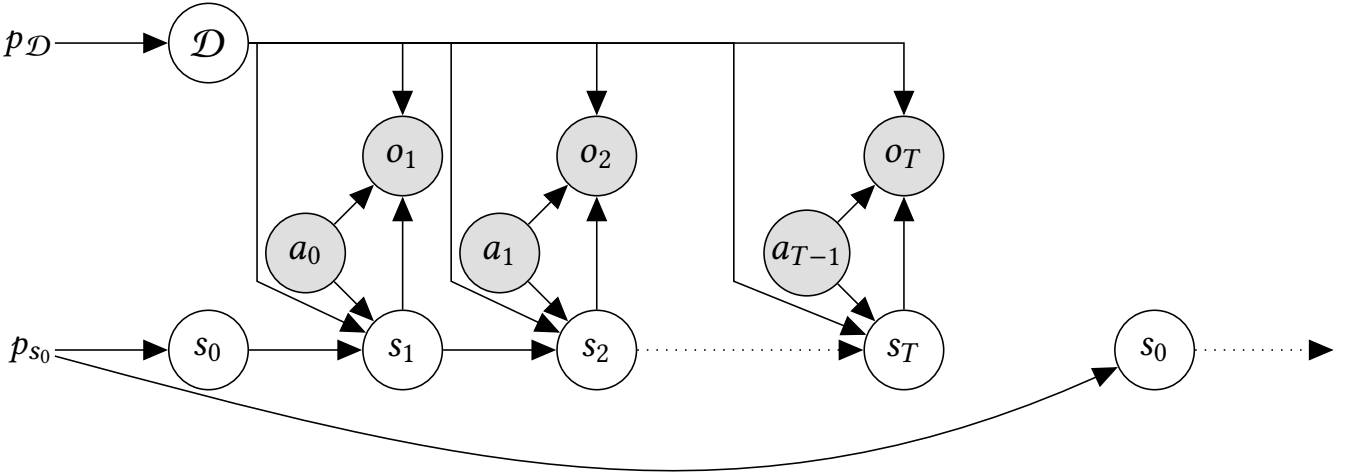
\begin{figure}[H]
    \centering
    \resizebox{\linewidth}{!}{{\input{figures/brl_graph}}}
    \caption{Graphical model of the BRL inference problem. The actions $a$ and
    observations $o$ depicted in gray are observable, meaning that the policy
    can be (is) dependent on them, while the states $s$ and dynamics
    $\mathcal{D}$ are hidden. The priors $p_\mathcal{D}$ and $p_{s_0}$
    represent the a-priori knowledge. Time is indicated with subscripts and
    progresses to the right.}\label{app-fig:brl-graph}
\end{figure}

% notation
where we, again, denote the observable history, the previous actions and
observations, with $\hist = (\at, \ot) = (a_0, o_1, a_1, \dots a_{t-1}, o_t)$,
as well as the priors $\modelPrior \in \Delta \mathcal{D}$ and $\statePrior \in
\Delta \mathbb{S}$ (\cref{app-fig:brl-graph}). We also continue to adopt the
notation for a full history as the combination of the observable history
$\hist$ with a state sequence $\st = s_0, s_1 \dots s_t$, as $\Hist =
(s_0,a_0,s_1,o_1 \dots a_{t-1},s_t,o_t)$:

% derivation of the theorem
\newcommand{\EtaAndSums}{\eta \sum_{s_t}\sum_{\stmin}}

\begin{align*}
    \text{belief} &= \sum_{\st} p(\stplus | \histplus) \modelDistrtplus \\
    (split\text{ }seq) &=
    \sum_{\st} p(\st,s_{t+1} | \hist,a_t,o_{t+1}) \modelDistrtplus \\
    (cond.) &= \sum_{\st}
    \frac{p(\st,s_{t+1},o_{t+1} | \hist, a_t)}{p(o_{t+1} | \hist, a_t)}
    \modelDistrtplus \\
    (chain) &= \eta \sum_{\st}
    p(\st | \hist, \cancel{a_t}) p(s_{t+1},o_{t+1} | \hist, \st, a_t)
    \modelDistrtplus \\
    (d \text{-}sep) &=  \eta \sum_{\st}
    p(\st | \hist) p(s_{t+1},o_{t+1} | \hist, \st, a_t) \modelDistrtplus \label{app-eq:d-sep} \\
    (\Hist = \st,\hist) &= \eta \sum_{\st}
    p(\st | \hist) p(s_{t+1},o_{t+1} | \Hist, a_t) \modelDistrtplus
\end{align*}

% D-separation explanation
The validity of the d-separation is shown in~\cref{app-fig:brl-graph}. The
influence of the action $a_t$ on the previous state sequence $\st$ is blocked
by the hidden observation $o_{t+1}$ and state $s_{t+1}$. Hence $p(\st | \hist,
a_t) = p(\st | \hist)$. \qed

%% file: appendix/hist-equivalence.tex
\newcommand{\bporlhmdp}{^\text{BRL}}
\newcommand{\gbapomdphmdp}{^\text{GBA}}

\newcommand{\transitionModel}{\mathcal{T}}

%%% SETUP

% outline
This section proves the equivalence between the history-MDP constructed from
the Bayesian PORL $\bporl^{\hmdp}$ problem definition and the history-MDP
constructed from the general BA-POMDP $\gbapomdp^{\hmdp}$ (theorem 1).

\subsubsection{Bayesian partially observable RL}

% BPORL assumptions
Recall that the BPORL problem is defined by (or assumes known) the state,
action and observation space $(\stateSpace, \actionSpace, \observationSpace)$
and a prior over the dynamics $\modelPrior$ and the initial state distribution
$\statePrior$.
% MDP assumptions
To consider optimizing reward, we additionally assume that the reward function
$\rewardFunction$, discount factor $\gamma$, and horizon $\horizon$ is known.
Lastly, we denote the history as the (past) observable action-observation
sequence $\hist=(\at,\ot)=(a_{0},o_{1},a_{1},\dots a_{t-1},o_{t})$, and a full
(partially observable) history (including a state-sequence $\st$), as $\Hist =
    (\st, \hist) = (s_0,a_0,s_1,o_1 \dots a_{t-1},s_t,o_t)$, then at any point in
time the belief over the dynamics and current state is defined as (eq 3 in the
paper):

% BPORL BU
\begin{equation}\label{eq:bporl-bu}
    b\bporlhmdp(s',\mathcal D|h') triangleq p\bporlhmdp(s', \mathcal D | h,a,o) \propto \sum_s \mathcal D(s', o|s, a) b\bporlhmdp(\mathcal D, s|h)
\end{equation}

Note that it can also be defined as a (non-recursive) expectation over the complete past state sequence:

% BPORL BU
\begin{equation}\label{eq:bporl-belief-rolled-out}
    b\bporlhmdp(s_{t+1},\mathcal D|\histplus) = \sum_{\st}p\bporlhmdp(\stplus, \mathcal D | \histplus)
\end{equation}

% BPORL hist-MDP definition
\begin{definition}[$\bporl^{\hmdp}$]

    The history-MDP constructed from the BPORL problem can be defined by the
    following tuple $(\stateSpace\bporlhmdp, \actionSpace\bporlhmdp,
        \transitionModel\bporlhmdp, \rewardFunction\bporlhmdp,
        \statePrior\bporlhmdp, \gamma\bporlhmdp, \horizon\bporlhmdp)$, where:

\end{definition}

\begin{itemize}

    \item The action space, discount function and horizon is directly given by
          the BPORL formulation: $\actionSpace\bporlhmdp = \actionSpace$,
          $\gamma\bporlhmdp = \gamma$, $\horizon\bporlhmdp = \horizon$

    \item The state space is the set of histories $\stateSpace\bporlhmdp =
              \{h\}$

    \item The initial history state is always the empty sequence:

          \begin{equation}
              \statePrior\bporlhmdp = \begin{cases} 1, & \text{ if } s = \emptyset \\ 0 \end{cases}
          \end{equation}

    \item The (history) reward function is an expectation over the given
          (state-) reward function:

          \begin{align}
              \rewardFunction\bporlhmdp(s\bporlhmdp, a\bporlhmdp) & = \rewardFunction\bporlhmdp(h, a)                                                                         \\
                                                                  & = \sum_s \int_{\mathcal D} b\bporlhmdp(s,\mathcal D|h) \rewardFunction(s, a)  \label{eq:bporl-reward-exp}
          \end{align}

          where the belief $b\bporlhmdp$ is computed as in~\cref{eq:bporl-bu}

    \item The (history) transition function is completely determined by the
          observation probability:

          \begin{align}
              \transitionModel\bporlhmdp(s\bporlhmdp, a\bporlhmdp, s^{,\text{BRL}}) & = \transitionModel\bporlhmdp(h, a, h')                                  \\
                                                                                    & = p\bporlhmdp(h' | h, a) = \begin{cases}
                                                                                                                     p\bporlhmdp(o | h, a), & \text{ if }h' = hao \\
                                                                                                                     0
                                                                                                                 \end{cases}
          \end{align}

          where the observation probability is, again, an expectation over the
          state:

          \begin{equation}\label{eq:bporl-obs-exp}
              p\bporlhmdp(o | h, a) \triangleq  \sum_{s,s'} \int_{\mathcal D} b\bporlhmdp(\mathcal D,s|h) \mathcal D(s',o|s,a)
          \end{equation}

\end{itemize}

\subsubsection{The GBA-POMDP}

% GBA-POMDP hist-MDP definition
On the other hand, let us assume some parameterization $\ptheta \in \pTheta$ of
distributions over dynamics $p(\mathcal D; \ptheta)$, and an initial (prior)
parameters $\pthetaprior$ that match the dynamics priors:

\begin{equation}\label{eq:gbapomdp-prior-criterion}
    \int_{\mathcal D}\modelPrior(\mathcal D)\mathcal D(s_1,o_1|s_0,a_0) = p(s_1,o_1|\pthetaprior,s_0,a_0) \left ( = \int_{\mathcal D} p(\mathcal D | \pthetaprior) D(s_1,o_1|s_0,a_0) \right )
\end{equation}

and a parameter update function $\mathcal{U}$ such that new parameters induce
the same distribution over dynamics as the history would (the \emph{parameter
    update criterion} holds:

\begin{definition}[Parameter update criterion]\label{def:app-parameter-update-criterion}
    $\mathcal U$ is designed such that, if

    \begin{equation}
        p\bporlhmdp(s_{t+1}, o_{t+1}|\Hist, a_t) = p\gbapomdphmdp(s_{t+1}, o_{t+1}|\pthetaH, s_{t}, a_t)
    \end{equation}

    then for $\pthetaHplus = \mathcal U(\pthetaH, s_t, a_t, s_{t+1}, o_{t+1})$,

    \begin{equation*}
        p\bporlhmdp(s_{t+2},o_{t+2}|\Histplus,s_{t+1},a_{t+1})=p\gbapomdphmdp(s_{t+2},o_{t+2}|\pthetaHplus,s_{t+1},a_{t+1})
    \end{equation*}

\end{definition}

Additionally, let us define the POMDP as a result of the definition of the
GBA-POMDP: $\gbapomdp^{\hmdp} (\pthetaprior, \mathcal U) = (
    \mathbb{\bar{S},A,O},\mathcal{\bar{D}, \bar{\rewardFunction}}, \gamma,
    \horizon, \gbaprior)$. Then the belief in this POMDP (as given in eq 1 in the
paper) is computed as follows:

% GBA-POMDP BU
\begin{equation}\label{eq:gbapomdp-bu}
    b\gbapomdphmdp(\bar s' = (s', \ptheta') | h') \triangleq p\gbapomdphmdp(s', \ptheta' | h,a,o) \propto \sum_{\bar s} \mathcal{\bar D}(s', \ptheta', o|s, \ptheta, a) b\gbapomdphmdp(s, \ptheta | h)
\end{equation}

where

\begin{equation}\label{eq:app-bapomdp-dynamics}
    \mathcal{\bar{\mathcal{D}}}(\ptheta',s',o|s,\ptheta,a) \triangleq
    p(\ptheta'|\ptheta,s,a,s',o) p(s',o|\ptheta,s,a)
\end{equation}

\begin{definition}[$\gbapomdp^{\hmdp}$]

    The history-MDP constructed from a GBA-POMDP definition can be defined by
    the following tuple $(\stateSpace\gbapomdphmdp, \actionSpace\gbapomdphmdp,
        \transitionModel\gbapomdphmdp, \rewardFunction\gbapomdphmdp,
        \statePrior\gbapomdphmdp, \gamma\gbapomdphmdp, \horizon\gbapomdphmdp)$,
    where:

\end{definition}

\begin{itemize}

    \item The action space, discount function and horizon is directly given by
          the BPORL formulation: $\actionSpace\gbapomdphmdp = \actionSpace$,
          $\gamma\gbapomdphmdp = \gamma$, $\horizon\gbapomdphmdp = \horizon$

    \item The state space is the set of histories $\stateSpace\gbapomdphmdp =
              \{h\}$

    \item The initial history state is always the empty sequence:

          \begin{equation}
              \statePrior\gbapomdphmdp = \begin{cases}
                  1, & \text{ if } s\gbapomdphmdp = \emptyset \\
                  0
              \end{cases}
          \end{equation}

    \item The (history) reward function is an expectation over the given
          (state-) reward function:

          \begin{align}
              \rewardFunction\gbapomdphmdp(s\gbapomdphmdp, a\gbapomdphmdp) & = \rewardFunction\gbapomdphmdp(h, a)                                                                   \\
                                                                           & = \sum_{\bar s}b\gbapomdphmdp(\bar s| h) \bar{\rewardFunction}(\bar s, a)                              \\
                                                                           & = \sum_{s, \ptheta} b\gbapomdphmdp(s, \ptheta| h) \rewardFunction(s, a) \label{eq:gbapomdp-reward-exp}
          \end{align}

          where the belief $\bar b\gbapomdphmdp$ is computed as
          in~\cref{eq:gbapomdp-bu}

    \item The (history) transition function is completely determined by the
          observation probability:

          \begin{align}
              \transitionModel\gbapomdphmdp(s\gbapomdphmdp, a\gbapomdphmdp, s^{,\text{GBA}}) & = \transitionModel\gbapomdphmdp(h, a, h')                                     \\
                                                                                             & = p\gbapomdphmdp(h' | h, a) = \begin{cases}
                                                                                                                                 p\gbapomdphmdp(o | h, a), & \text{ if }h' = hao \\
                                                                                                                                 0
                                                                                                                             \end{cases}
          \end{align}

          where the observation probability is, again, an expectation over the
          state:

          \begin{equation} \label{eq:gbapomdp-obs-exp}
              p\gbapomdphmdp(o | h, a) \triangleq \sum_{s,\ptheta,s',\ptheta'} b\gbapomdphmdp(s,\theta | h) \bar{\mathcal D}(\theta',s',o|s,\theta,a)
          \end{equation}

\end{itemize}

\subsubsection{Trivial part of proof}

The proof consists of showing that each element of the belief-MDP tuples are
identical. It is clear from the definition/construction that the state space,
action space, state prior, discount factor and horizon are the same:

\begin{itemize}
    \item $\stateSpace\bporlhmdp = \stateSpace\gbapomdphmdp = \{h\}$
    \item $\actionSpace\bporlhmdp = \actionSpace\gbapomdphmdp = \actionSpace$
    \item $\statePrior\bporlhmdp = \statePrior\gbapomdphmdp = \begin{cases}
                  1, & \text{ if } s\bporlhmdp = \emptyset \\
                  0
              \end{cases}
          $
    \item $\gamma\bporlhmdp = \gamma\gbapomdphmdp = \gamma$
    \item $\horizon\bporlhmdp = \horizon\gbapomdphmdp = \horizon$
\end{itemize}

All that is left to proof is the equivalence of the transition and reward
function (i.e.~\cref{eq:bporl-reward-exp} equals~\cref{eq:gbapomdp-reward-exp}
and~\cref{eq:bporl-obs-exp} equals~\cref{eq:gbapomdp-obs-exp}).

To prove that the reward and transition model of the two history-MDPs are the
same we proof that for all histories $h$ the distribution over the state
sequence $p(\st | \hist)$ and observation $p(o | \hist, a)$ is the same. The
proof is via induction on both quantities.

Note that in these derivations $p\bporlhmdp$ refers to quantities in the BPORL
history-MDP $\bporl$, and $p\gbapomdphmdp$ analogously for the GBA-POMDP
history-MDP $\gbapomdp$.

\subsection{Proof State-Sequence Distribution Equality}

To prove:

\begin{equation}
    p\bporlhmdp(\stplus | \histplus) = p\gbapomdphmdp(\stplus | \histplus)
\end{equation}

where

\begin{align}
    p\bporlhmdp(\stplus | \histplus)    & = \int_{\mathcal D} p\bporlhmdp(\stplus, \mathcal D| \histplus) \\
    p\gbapomdphmdp(\stplus | \histplus) & = \int_{\ptheta'} p\gbapomdphmdp(\stplus, \ptheta'| \histplus)
\end{align}

\subsubsection{Base case: $t=0$}

Holds as for both models the prior is defined by $\statePrior$:

\begin{align}
    \int_{\mathcal D} b\bporlhmdp(s_0, \mathcal D| \hist{=}\{\}) & = \int_{\mathcal D} b\gbapomdphmdp(s_0, \mathcal D| \hist{=}\{\}) \\
    \statePrior                                                  & = \statePrior
\end{align}

\subsubsection{Induction case: $t+1$ given it holds for $t$}

Here we derive the state sequence distribution for both sides for $t+1$.

\paragraph{BPORL (lhs) term}

Note that the derivation of the first step is given in this
appendix (\cref{ssec:proofs:brl-derivation}).

\begin{align}
                                       & \int_{\mathcal D} p(\stplus, \mathcal D| \histplus)                                                                                                                       \\
    (\ref{ssec:proofs:brl-derivation}) & = \int_{\mathcal D} \frac{p\bporlhmdp(\st| \hist)}{p\bporlhmdp(o_{t+1}|\hist, a)} p\bporlhmdp(s_{t+1}, o_{t+1} | \Hist, a_t) p\bporlhmdp(\mathcal D | \Histplus)          \\
    (move\ \int)                       & =  \frac{p\bporlhmdp(\st| \hist)}{p\bporlhmdp(o_{t+1}|\hist, a)} p\bporlhmdp(s_{t+1}, o_{t+1} | \Hist, a_t) \cancel{\int_{\mathcal D}p\bporlhmdp(\mathcal D | \Histplus)} \\
    (tot\ prob)                        & =  \frac{p\bporlhmdp(\st| \hist)}{p\bporlhmdp(o_t|\hist, a)} p\bporlhmdp(s_{t+1}, o_t | \Hist, a_t) \label{eq:bporl-state-sequence-derived}
\end{align}

\paragraph{GBA-POMDP (rhs) term}

And we see that we can reach the same in the GBA-POMD:

\begin{align}
    \int_{\ptheta'} p\gbapomdphmdp(\stplus, \ptheta'| \histplus) & = \int_{\ptheta,\ptheta'} p\gbapomdphmdp(\stplus,\theta,\theta'|\hist,a_t,o_{t+1})                                                                                                                                            \\
    (cond)                                                       & = \int_{\ptheta,\ptheta'} \frac{p\gbapomdphmdp(\st,s_{t+1},\theta,\theta',o_{t+1}|\hist,a_t)}{p\gbapomdphmdp(o_{t+1}|\hist,a_t)}                                                                                              \\
    (chain)                                                      & = \int_{\ptheta,\ptheta'} \frac{p\gbapomdphmdp(\st,\theta|\hist,\cancel{a_t})}{p\gbapomdphmdp(o_{t+1}|\hist,a_t)} p\gbapomdphmdp(s_{t+1},o_{t+1},\theta' | \st, \theta, \cancel{\hist}, a_t)                                  \\
    (d \text{-} sep)                                             & = \int_{\ptheta,\ptheta'} \frac{p\gbapomdphmdp(\st,\theta|\hist)}{p\gbapomdphmdp(o_{t+1}|\hist,a_t)} \bar{\mathcal D}(s_{t+1},o_{t+1},\theta' | s_t, \theta, a_t)                                                             \\
    (\ref{eq:app-bapomdp-dynamics})                              & = \int_{\ptheta,\ptheta'} \frac{p\gbapomdphmdp(\st,\theta|\hist)}{p\gbapomdphmdp(o_{t+1}|\hist,a_t)}         \identityF(\ptheta',\mathcal{U}(\ptheta,s_t,a_t,s_{t+1},o_{t+1}))p\gbapomdphmdp(s_{t+1},o_{t+1}|\ptheta,s_t,a_t) \\
                                                                 & = \frac{p\gbapomdphmdp(\st|\hist)}{p\gbapomdphmdp(o_{t+1}|\hist,a_t)} p\gbapomdphmdp(s_{t+1},o_{t+1}|\ptheta_{\Hist},s_t,a_t) \label{eq:gba-pomdp-state-sequence-derived}
\end{align}

The integrals disappear because all probability is concentrated on just one set
of parameters $(\theta = \pthetaH, \theta'=\pthetaHplus)$: there is only
one parameter set associated with a particular history ($\int_{\theta}
    p\gbapomdphmdp(\st,\theta|\hist) = 0$ for all except $\pthetaH$), and a
deterministic update to $\pthetaHplus$. \\

Note that
equations~\cref{eq:bporl-state-sequence-derived,eq:gba-pomdp-state-sequence-derived}
are equal if

\begin{enumerate}

    \item the parameter update criterion holds $p\bporlhmdp(s_{t+1}, o_t |
              \Hist, a_t) = p\gbapomdphmdp(s_{t+1},o_{t+1}|\ptheta_{\Hist},s_t,a_t)$

    \item the state sequence probability at $t$ is equal $p\bporlhmdp(\st|
              \hist) = p\gbapomdphmdp(\st|\hist)$, and

    \item the observation probability is equal
          $p\bporlhmdp(o_{t+1}|\hist, a) = p\gbapomdphmdp(o_{t+1}|\hist,a_t)$

\end{enumerate}

Item 1 is assumed, item 2 is proven through induction, and item 3 is proven
below (\cref{ssec:observation-prob-proof}):

\begin{align}
    \cref{eq:bporl-state-sequence-derived}                                                 & = \cref{eq:gba-pomdp-state-sequence-derived}                                                                     \\
    \frac{p\bporlhmdp(\st| \hist)}{p\bporlhmdp(o_t|\hist, a)} p(s_{t+1}, o_t | \Hist, a_t) & = \frac{p\gbapomdphmdp(\st|\hist)}{p\gbapomdphmdp(o_{t+1}|\hist,a_t)} p(s_{t+1},o_{t+1}|\ptheta_{\Hist},s_t,a_t)
\end{align}

\qedsymbol

\subsection{Proof Observation Probability Equality}\label{ssec:observation-prob-proof}

To prove:

\begin{equation}
    p\bporlhmdp(o_{t+1}|\hist,a_t) = p\gbapomdphmdp(o_{t+1}|\hist, a_t)
\end{equation}

where, by definitions (\cref{eq:bporl-obs-exp,eq:gbapomdp-obs-exp}):

\begin{align}
    p\bporlhmdp(o_{t+1}|\hist,a_t)     & = \sum_{s,s'} \int_{\mathcal D} b\bporlhmdp(\mathcal D,s|h) \mathcal D(s',o|s,a)                      \\
    p\gbapomdphmdp(o_{t+1}|\hist, a_t) & = \sum_{s,\ptheta,s',\ptheta'} b\gbapomdphmdp(s,\theta | h) \bar{\mathcal D}(\theta',s',o|s,\theta,a)
\end{align}

\paragraph{BPORL (lhs) term}

We again first derive the lhs:

\begin{align}
                                       & \sum_{s,s'} \int_{\mathcal D} b\bporlhmdp(\mathcal D,s|h) \mathcal D(s',o|s,a)                                                                             \\
    (\ref{eq:bporl-belief-rolled-out}) & = \sum_{\stplus} \int_{\mathcal D} p\bporlhmdp(\st, \mathcal D|\hist) \mathcal D(s_{t+1},o_{t+1}|s_t,a_t)                                                  \\
    (chain)                            & = \sum_{\stplus} \int_{\mathcal D} p\bporlhmdp(\st|\hist) p\bporlhmdp(\mathcal D | \Hist) \mathcal D(s_{t+1},o_{t+1}|s_t,a_t)                              \\
    (move\ \int)                       & = \sum_{\stplus} p\bporlhmdp(\st|\hist) \int_{\mathcal D} p\bporlhmdp(\mathcal D | \Hist) \mathcal D(s_{t+1},o_{t+1}|s_t,a_t)                              \\
                                       & = \sum_{\stplus} p\bporlhmdp(\st|h) p\bporlhmdp(s_{t+1}, o_{t+1} | \Hist, a_t)                                                \label{eq:bporl-obs-derived}
\end{align}

\paragraph{GBA-POMDP (rhs) term}

and show we can reach the same equality with the rhs:

\begin{align}
                                               & \sum_{s,\ptheta,s',\ptheta'} b\gbapomdphmdp(s,\theta | h) \bar{\mathcal D}(\theta',s',o|s,\theta,a)                                                                                                        \\
    (tot\ prob)                                & = \sum_{\stplus,\ptheta,\ptheta'} p\gbapomdphmdp(\st,\theta | \hist) \bar{\mathcal D}(\theta',s_{t+1},o_{t+1}|s_t,\ptheta,a_t)                                                                             \\
    (\ref{eq:app-bapomdp-dynamics})            & = \sum_{\stplus,\ptheta,\ptheta'} p\gbapomdphmdp(\st,\theta | \hist) \identityF(\ptheta',\mathcal{U}(\ptheta,s_t,a_t,s_{t+1},o_{t+1}))p(s_{t+1},o_{t+1}|\ptheta,s_t,a_t)                                   \\
    (\ref{def:app-parameter-update-criterion}) & = \sum_{\stplus} p\gbapomdphmdp(\st | \hist) p(s_{t+1}, o_{t+1}|s_t, \theta_{\Hist}, a_t)                                                                                 \label{eq:gba-pomdp-obs-derived}
\end{align}

where again the integrals disappear because all probability is concentrated on
one set of parameters. Equality
between~\cref{eq:bporl-obs-derived,eq:gba-pomdp-obs-derived} is ensured when
the parameter update criterion is met and the state sequence distribution is
the same (the same requirements as in the enumeration above). Since those have
been proven above, this finishes the proof:

\begin{align}
    \cref{eq:bporl-obs-derived}                                                  & = \cref{eq:gba-pomdp-obs-derived}                                                                      \\
    \sum_{\stplus} p\bporlhmdp(\st|h) p\bporlhmdp(s_{t+1}, o_{t+1} | \Hist, a_t) & = \sum_{\stplus} p\gbapomdphmdp(\st | \hist) p\gbapomdphmdp(s_{t+1}, o_{t+1}|s_t, \theta_{\Hist}, a_t)
\end{align}

\qedsymbol

%% file: appendix/lanes.tex
Most domains are taken directly from the literature
(tiger~\cite{kaelbling_planning_1998} and collision
avoidance~\cite{luo_importance_2019}. Gridverse is an instantiation of the
`empty-` domains created in \url{https://github.com/abaisero/gym-gridverse}.
Road racing is a novel domain and thus deserves a little more detailed
description.

% problem description
In the road racing problem, a larger POMDP grid model of highway traffic, the
agent moves between $n$ lanes in an attempt to overtake other
cars~\ref{fig:rr-vis}. The agent occupies one of the lanes, and can switch
between them by going up, down or otherwise stay in the current. The agent
shares the road with $n$ other cars, one in each lane. The (partially hidden)
state is described by the distance of each of those cars in their respective
lanes (horizontally) and the current occupied lane. During a step the distance
of the other cars decrements with some probability. If the agent either
attempts to leave the grid or bumps against another car it receives a penalty
(-1) and stays in its current lane instead. The agent observes only the
distance to the car in its current lane, which is also the reward (plus
potentially the penalty mentioned above). The speed, and thus the probability
of a car coming closer depends on the lane, and is a linear interpolation over
the lanes (probability of lane $i$ is $p_i = \frac{i+1}{n+1}$). Lastly, the
initial position of all cars is 6, and when the position of a car drops to -1,
it resets. More specifically:

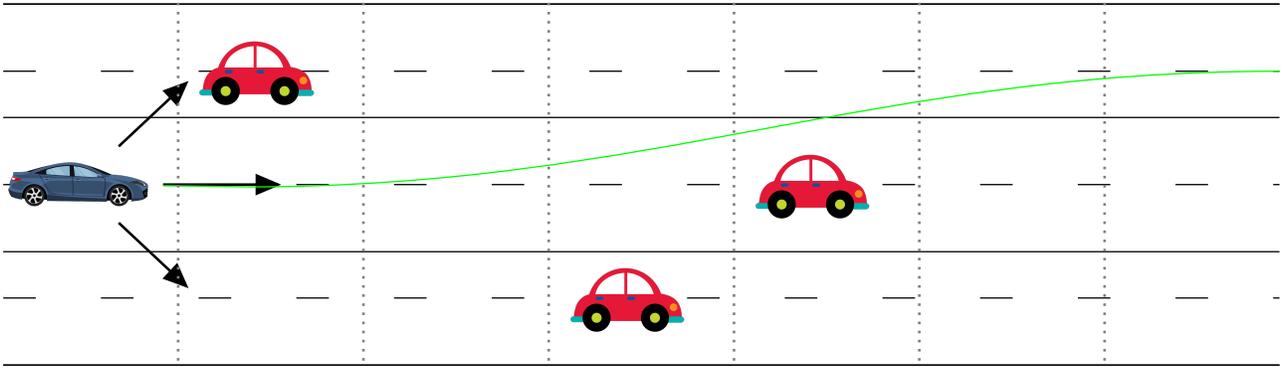
\begin{figure}[htpb]
    \centering
    \caption{Road racer domain with 3 lanes}%
    \resizebox{1\textwidth}{!}{ \input{figures/road-racer-vis.tex} }
    \label{fig:rr-vis}
\end{figure}

% road race properties
\begin{itemize}

    \item The domain is parameterized by number of lanes $n$ and max distance
          of cars $l$, assumed to be 6 here

    \item State space $S$ consists of the current lane of agent and position of
          cars in each lane: $|S| = n \times l^n$

    \item Action space $A$ is $\{up, stay, down\}$, the action moves the agent
          deterministically to the adjacent lanes (up or down), or stays in place. If
          moving results in a collision with a car (or an attempt is made to go
          beyond the most extreme lanes on either side) then a penalty (-1) is
          incurred.

    \item Observation space $\observationSpace$ is $\{ 0 \dots l \}$, the
          observation $\mathcal{O}$ deterministically the distance to the car in the
          current lane

    \item The reward function $\rewardFunction$ returns the distance to the car
          in the current lane. A penalty of -1 is given if the agent either attempts
          to move into a car or out of bound

    \item The transition function $\mathcal{T}$ first updates the location of
          the car in each lane by decrementing their position ('distance from the
          agent') with some probability. Then the agent's location is updated by
          moving it deterministically according to its action, but staying in place
          if a move (up or down) action fails due to collision or out of bound. Cars
          that were overtaken by the agent (i.e.\ their position dropped below 0)
          have their position updated by setting it to $l$ (the next car appears).

\end{itemize}

\Cref{fig:rr-bn} shows, for arbitrary ($n$) number of lanes, the (smallest)
Bayes-network that is able to model the road racer domain correctly.

\begin{figure}[htpb]
    \centering
    \caption{Bayes-network representing the road-racer domain. $l$ refers to the lane (row) that is occupied by the agent, $c_i$ is the (column) position of car $i$. The observation is the distance/position $dist$ of the car in the agent's lane $l$}%
    \resizebox{.5\textwidth}{!}{ \input{figures/road-racer-bayes-net.tex} }
    \label{fig:rr-bn}
\end{figure}

%% file: figures/road-racer-vis.tex
\begin{tikzpicture}[every node/.style={scale=1}]

    % center lane
    \node (center-start) {};
    \node (center-end) at (14, 0) {};

    % other lanes
    \node[below=of center-start] (bottom-start) {};
    \node[above=of center-start] (top-start) {};

    \node[below=of center-end] (bottom-end) {};
    \node[above=of center-end] (top-end) {};

    % corners
    \node[above=.5 of top-start] (top-left-corner) {};
    \node[below=.5 of bottom-start] (bottom-left-corner) {};

    \node[above=.5 of top-end] (top-right-corner) {};
    \node[below=.5 of bottom-end] (bottom-right-corner) {};

    % top/bottom edge
    \node[above=.5 of center-start] (top-lane-border-start) {};
    \node[below=.5 of center-start] (bottom-lane-border-start) {};

    \node[above=.5 of center-end] (top-lane-border-end) {};
    \node[below=.5 of center-end] (bottom-lane-border-end) {};

    % lanes
    \draw [dash pattern= on 10pt off 20pt] (center-start) -- (center-end);
    \draw [dash pattern= on 10pt off 20pt] (top-start) -- (top-end);
    \draw [dash pattern= on 10pt off 20pt] (bottom-start) -- (bottom-end);

    % lane border
    \draw (top-left-corner) -- (top-right-corner);
    \draw (bottom-left-corner) -- (bottom-right-corner);

    \draw (top-lane-border-start) -- (top-lane-border-end);
    \draw (bottom-lane-border-start) -- (bottom-lane-border-end);

    \draw (top-left-corner) -- (top-right-corner);
    \draw (bottom-left-corner) -- (bottom-right-corner);

    % agent
    \node[right] at (center-start) (agent) {\includegraphics[height=.6cm]{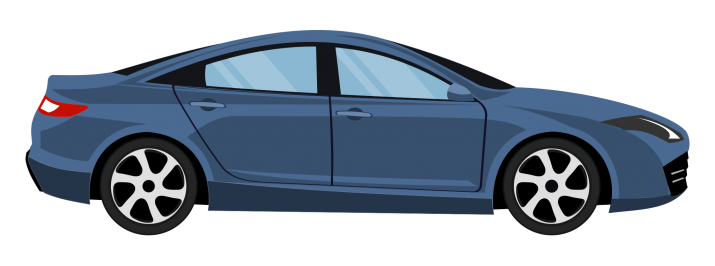}};

    % other cars
    \node[right=2 of top-start] (car-1) {\includegraphics[height=1.25cm]{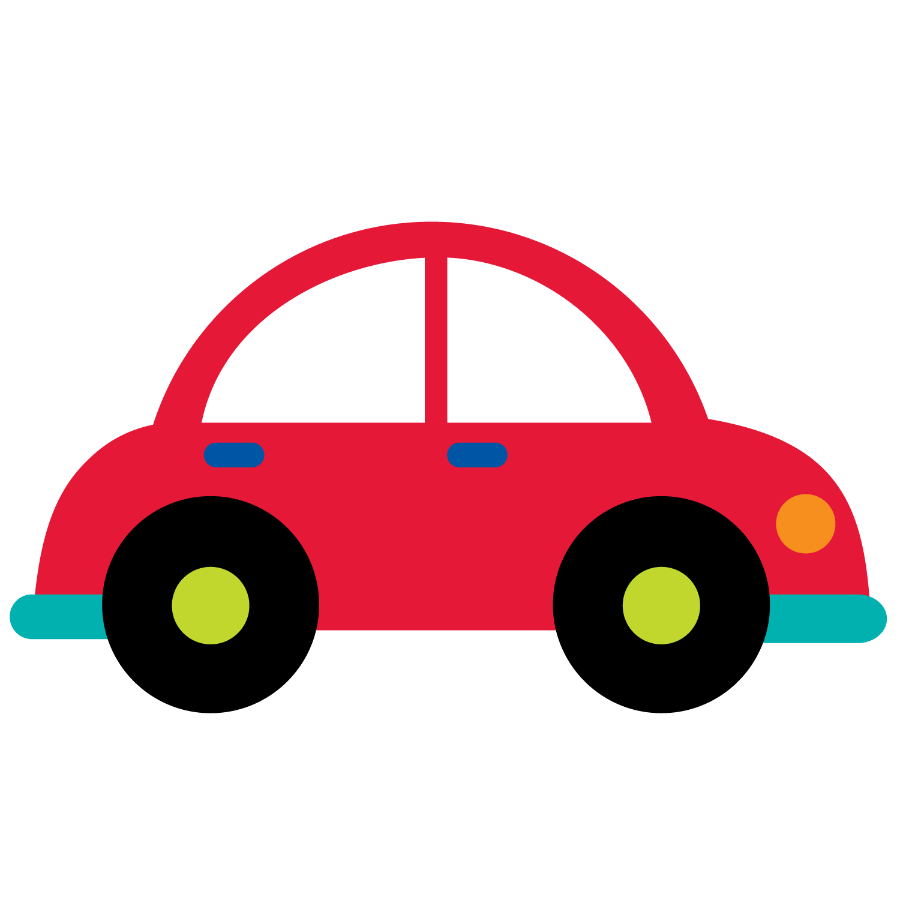}};
    \node[right=8 of center-start] (car-2) {\includegraphics[height=1.25cm]{figures/embeddings/car-1}};
    \node[right=6 of bottom-start] (car-3) {\includegraphics[height=1.25cm]{figures/embeddings/car-1}};

    % actions
    \node[right=2 of top-start] (action-up) {};
    \node[right=3 of center-start] (action-stay) {};
    \node[right=2 of bottom-start] (action-down) {};

    \node[below=.4 of car-1] (traj-break) {};

    \draw[thick,->] (agent) -- (action-up);
    \draw[thick,->] (agent) -- (action-stay);
    \draw[thick,->] (agent) -- (action-down);

    \draw[green] (agent) -- (traj-break);
    \draw[green] (traj-break.west) to[out=0,in=180] (top-end);

    % grid
    \foreach \i in {2,4,6,8,10,12,14}
        {
            \draw[thick,dotted,color=gray] (\i, -2) -- (\i, 2);
        }

\end{tikzpicture}

%% file: figures/road-racer-bayes-net.tex
\begin{tikzpicture}

    % define rows

    \node (s) {};
    \node[right=of s,obs] (a) {$a$};
    \node[right=2 of a] (s') {};
    \node[right=2 of s'] (o) {};

    % agent location
    \node[latent,below=of s] (l) {$l$};
    \node[latent,below=of s'] (l') {$l'$};
    \draw[->] (l) -- (l');

    % observation
    \node[obs,below=of o] (dist) {$dist$};
    \draw[->] (l') -- (dist);

    % first cars
    \node[latent,below=of l] (c-0) {$c_0$};
    \node[latent,below=of l'] (c'-0) {$c'_0$};

    \draw[->] (c-0) -- (c'-0);
    \draw[->] (c-0) -- (l');
    \draw[->] (c'-0) -- (dist);

    \draw[->] (a) -- (l');
    \draw[->] (a) -- (c'-0);
    \draw[->] (a) -- (dist);

    \foreach \i in {1,...,4}
        {
            \pgfmathtruncatemacro{\prev}{\i - 1}
            % state variable
            \node[latent,below=.4 of c-\prev] (c-\i) {$c_\i$};
            % next state variable
            \node[latent,below=.4 of c'-\prev] (c'-\i) {$c'_\i$};

            % car -> car'
            \draw[->] (c-\i) -- (c'-\i);

            % car -> agent'
            \draw[->] (c-\i) -- (l');

            % car' -> obs
            \draw[->] (c'-\i) -- (dist);

            % action -> car'
            \draw[->] (a) -- (c'-\i);
        }

    % dotted
    \node[below=.1 of c-4] (c-dots) {$\vdots$};
    \node[below=.1 of c'-4] (c'-dots) {$\vdots$};

    % 'rest' factors
    \node[latent,below=.2 of c-dots] (c-rest) {$c_n$};
    \node[latent,below=.2 of c'-dots] (c'-rest) {$c'_n$};

    \draw[dotted,->] (c-rest) -- (c'-rest);
    \draw[dotted,->] (c-rest) -- (l');
    \draw[dotted,->] (c'-rest) -- (dist);

    \plate {s-plate} {(l) (c-0) (c-1) (c-2) (c-3) (c-4) (c-rest)} {$s$};
    \plate {s'-plate} {(l') (c'-0) (c'-1) (c'-2) (c'-3) (c'-4) (c'-rest)} {$s'$};
    \plate {o-plate} {(dist)} {$o$};

\end{tikzpicture}

%% file: appendix/parameters.tex
The (hyper) parameters of BADDr are conceptually separated into groups. First
(\cref{ssec:parameters:networks}) the networks that govern the architecture and
learning of the neural networks in the particles. Second
(\cref{ssec:parameters:solution}) are the parameters of the planning solution
methods MCTS and particle filtering. Lastly (\cref{ssec:parameters:experiment})
there are higher level parameters (such as horizon and discount factor) and
environment parameters.

\subsection{Neural Networks}\label{ssec:parameters:networks}

Here, again, we identify two groups of parameters. First is regarding the
general hierarchy (\cref{tab:parameters-nn}), and the other is specifically for
pre-training during the creation of the prior
(\cref{tab:parameters-nn-pretrain}).

\begin{table}[H]
    \centering
    \caption{Parameters describing BADDr pre-training}
    \label{tab:parameters-nn-pretrain}
    \begin{tabular}{l | c | c |c | c }
        domain                 & tiger      & collision avoidance & road racing                   & gridverse \\ \hline
        \# batches             & 4096       & 8192                & 2048 ($n=3$), 16384 ($n=9$)   & 512       \\
        batch size             & 32         & 32                  & 64 ($n=3$), 256 ($n=9$)       & 32        \\
        learning rate $\alpha$ & 0.1        & 0.05                & 0.005 ($n=3$), 0.0025 ($n=9$) & 0.0025    \\
        optimizer              & SGD        & SGD                 & SGD                           & Adam      \\
        \# pre-trained nets    & 1, 4, or 8 & 1                   & 1                             & 1
    \end{tabular}
\end{table}

\begin{table}[H]
    \centering
    \caption{Parameters describing BADDr neural networks}
    \label{tab:parameters-nn}
    \begin{tabular}{l | c | c |c | c }
        domain                        & tiger  & collision avoidance & road racing             & gridverse \\ \hline
        \# layers                     & 3      & 3                   & 3                       & 3         \\
        \# nodes per layer            & 32     & 64                  & 32 ($n=3$), 256 ($n=9$) & 256       \\
        activation functions          & $tanh$ & $tanh$              & $tanh$                  & $tanh$    \\
        online learning rate $\alpha$ & 0.005  & 0.0005              & 0.0001 ($n=9$)          & 0.0005    \\
        dropout probability           & 0.5    & 0.5                 & 0.1                     & 0.1
    \end{tabular}
\end{table}

\subsection{Solution Methods}\label{ssec:parameters:solution}

In our experience the method is robust to these parameters
(\cref{tab:solution-parameters}). For example, the method scales linearly both
in complexity and performance with the number of simulations and particles. A
notable exception is the exploration rate, which is set approximately to the
range of returns that can be expected from an episode.

\begin{table}[H]
    \centering
    \caption{Parameters used by the (planning) solution methods: MCTS and belief update}
    \label{tab:solution-parameters}
    \begin{tabular}{l | c | c |c | c }
        domain                    & tiger      & collision avoidance & road racing & gridverse \\ \hline
        exploration constant $u$  & 100        & 1000                & 15          & 1         \\
        belief update             & importance & importance          & rejection   & rejection \\
        IS resample size          & 128        & 32                  & N/A         & N/A       \\
        \# of particles in filter & 1024       & 128                 & 1024        & 8         \\
        \# MCTS simulations       & 4096       & 256                 & 128         & 8         \\
        search depth              & $\horizon$ & $\horizon$          & 3           & N/A
    \end{tabular}
\end{table}

\subsection{Experiment and Domain Parameters}\label{ssec:parameters:experiment}

These parameters (\cref{tab:experiment-parameters}) have little to do with the
method, but defines the domain and experiment setup more generally.

\begin{table}[H]
    \centering
    \caption{Parmeters more general to the experiment setup}
    \label{tab:experiment-parameters}
    \begin{tabular}{l | c | c |c | c }
        domain                      & tiger & collision avoidance & road racing               & gridverse \\ \hline
        \# episodes                 & 400   & 500                 & 200 ($n=3$) / 300 ($n=9$) & 250       \\
        \# horizon $\horizon$       & 30    & N/A                 & 20                        & 30        \\
        \# discount factor $\gamma$ & 0.95  & 0.95                & 0.95                      & 0.95      \\
        \# runs                     & 7500  & 35000               & 1000 ($n=3$), 300 ($n=9$) & 50
    \end{tabular}
\end{table}

%% file: appendix/dpfrl.tex
DPFRL~\cite{ma_discriminative_2020} uses a learned particle filter to summarize
the observation-action history. Unlike a traditional filter, each particle is
not a hypothesis of the true state but is a latent vector with no semantic
meaning. The particle maintains a number of particles, which will then be
summarized to create a summary vector acting as the current belief state. This
belief state is used to learn an actor-critic agent using A2C.  We use the
official implementation code at \url{https://github.com/Yusufma03/DPFRL}. For
Tiger, Road-Race, and Collision-Avoidance, where observations are single
categorical values, we simply transform them into one-hot vectors and apply 2
FC layers to encode. For Grid-Verse where observations are 3D arrays of
categorical values, we first apply a single Pytorch's embedding layer and then
use a CNN of 3 layers (32, 64, and 64 channels) to encode. In all domains,
actions are encoded with a single FC(64) layer. For each domain, we perform
hyper-parameter sweeps on the learning rates \{1, 3, 10\}$\times 10^{-4}$ and
the number of particles \{15, 30, 45\}. The best performing learning rates and
the number of particles are reported in Figure \ref{fig:tiger},
\ref{fig:collision}, \ref{fig:road}, \ref{fig:grid}; other hyper-parameters are
listed in Table \ref{tab:training-params}.

\begin{table}[htbp]
    \centering
    \begin{tabular}{l c c}
        \toprule
        \textbf{Name}                   & \textbf{Value}                      \\
        \midrule
        Number of actors                & 16                                  \\
        Sample length                   & 5                                   \\
        Critic loss coefficient         & 0.5                                 \\
        Actor loss coefficient          & 1.0                                 \\
        Entropy loss coefficient        & 0.01                                \\
        Clipped gradient norm magnitude & 5                                   \\
        Optimizer                       & RMSProp ($\epsilon=1\times10^{-5}$) \\
        \bottomrule
    \end{tabular}
    \vspace{0.1in}
    \caption{
        Hyper-parameters (except for learning rates and the number of particles
        which are tuned for each domain) values used in DPFRL in all domains.
    }
    \label{tab:training-params}
\end{table}

\begin{figure}
    \begin{subfigure}[t]{.45\linewidth}
        \centering
        \includegraphics[width=1\linewidth]{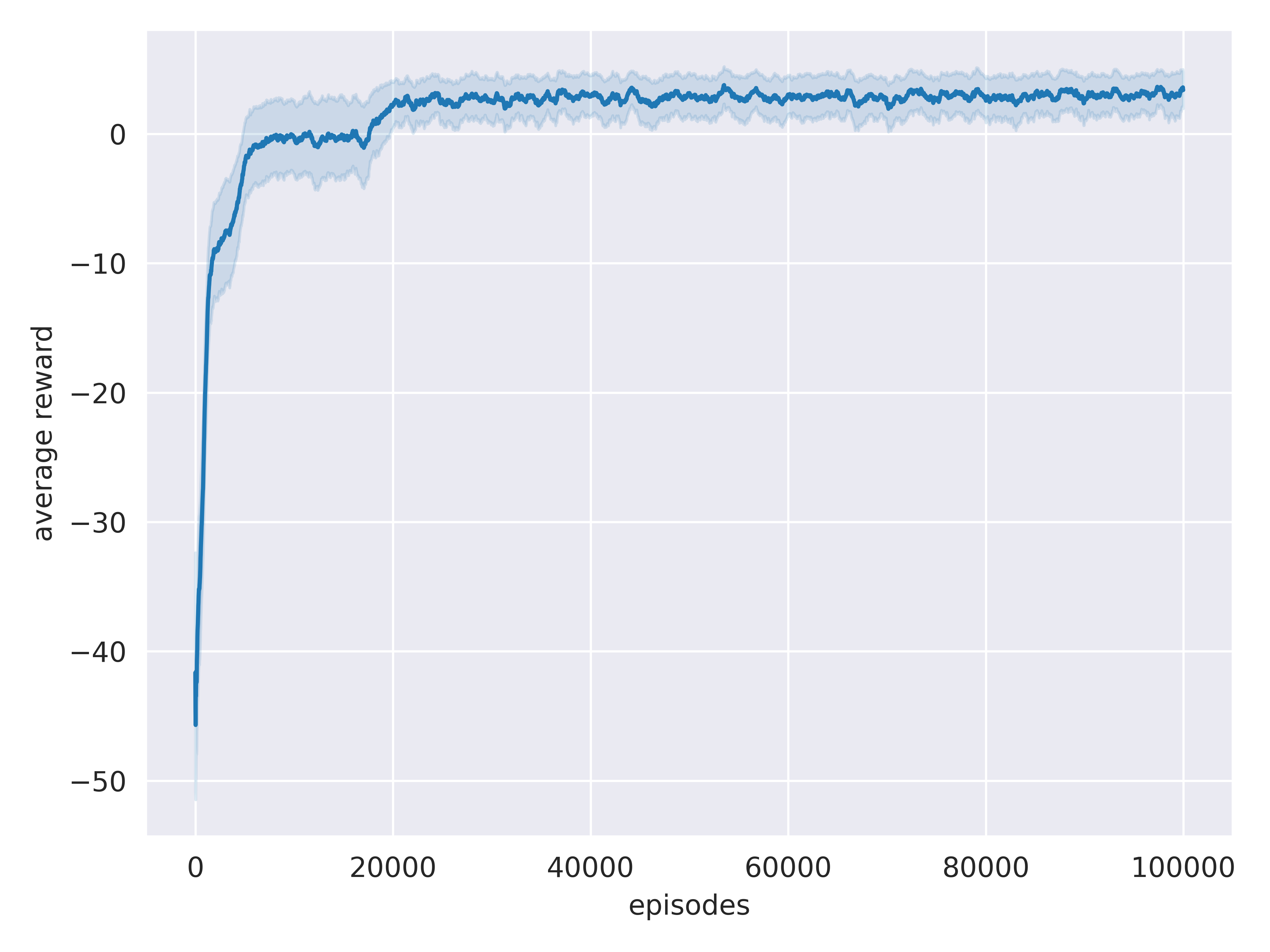}
        \caption{Return on Tiger (lr=1e-4, number of particles=30).}
        \label{fig:tiger}
    \end{subfigure}%
    \begin{subfigure}[t]{.45\linewidth}
        \centering
        \includegraphics[width=1\linewidth]{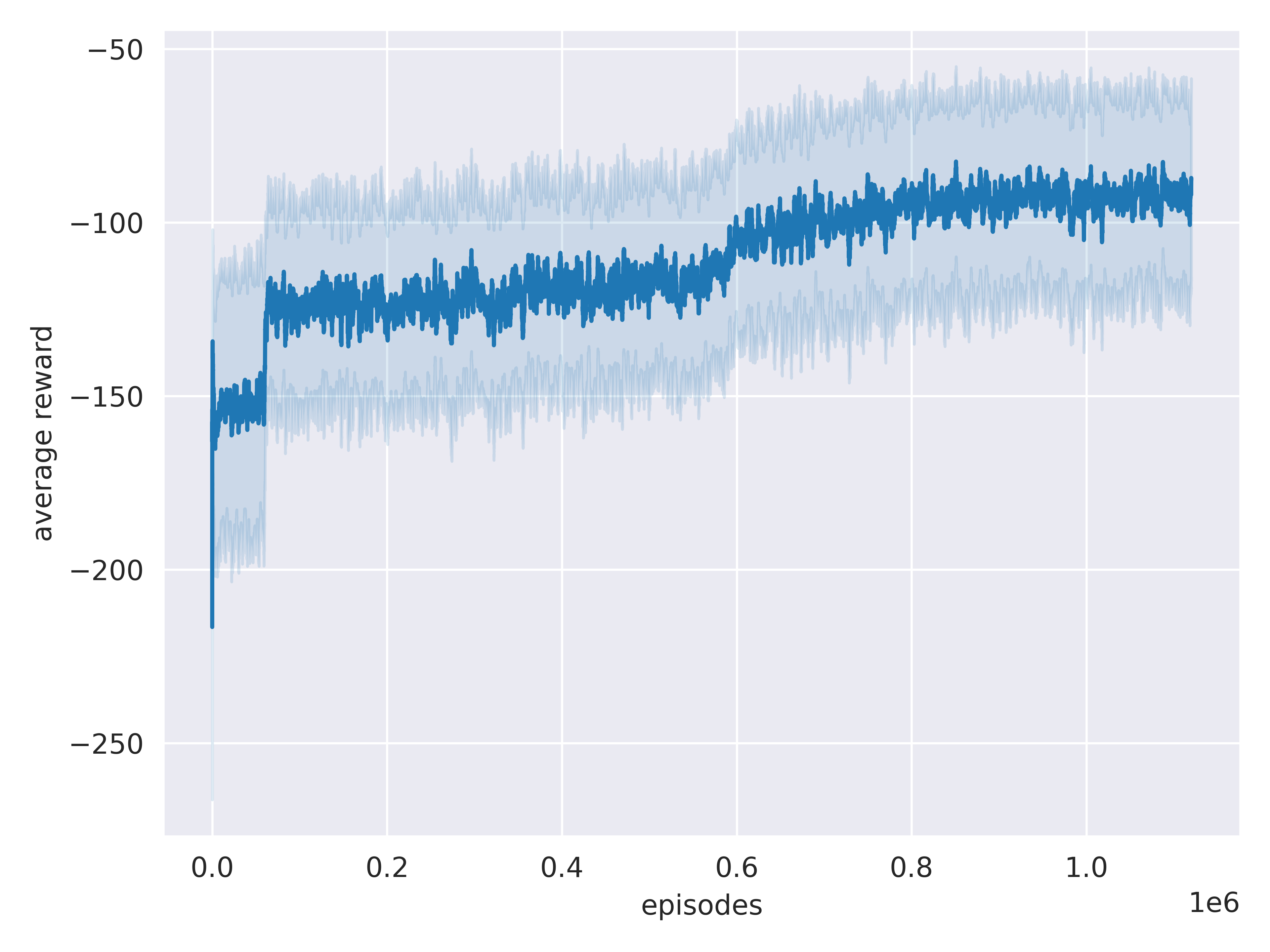}
        \caption{Return on Collision-Avoidance (lr=1e-4, number of particles=30).}\label{fig:collision}
    \end{subfigure}
\end{figure}

\begin{figure}
    \begin{subfigure}[b]{.45\linewidth}
        \centering
        \includegraphics[width=1\linewidth]{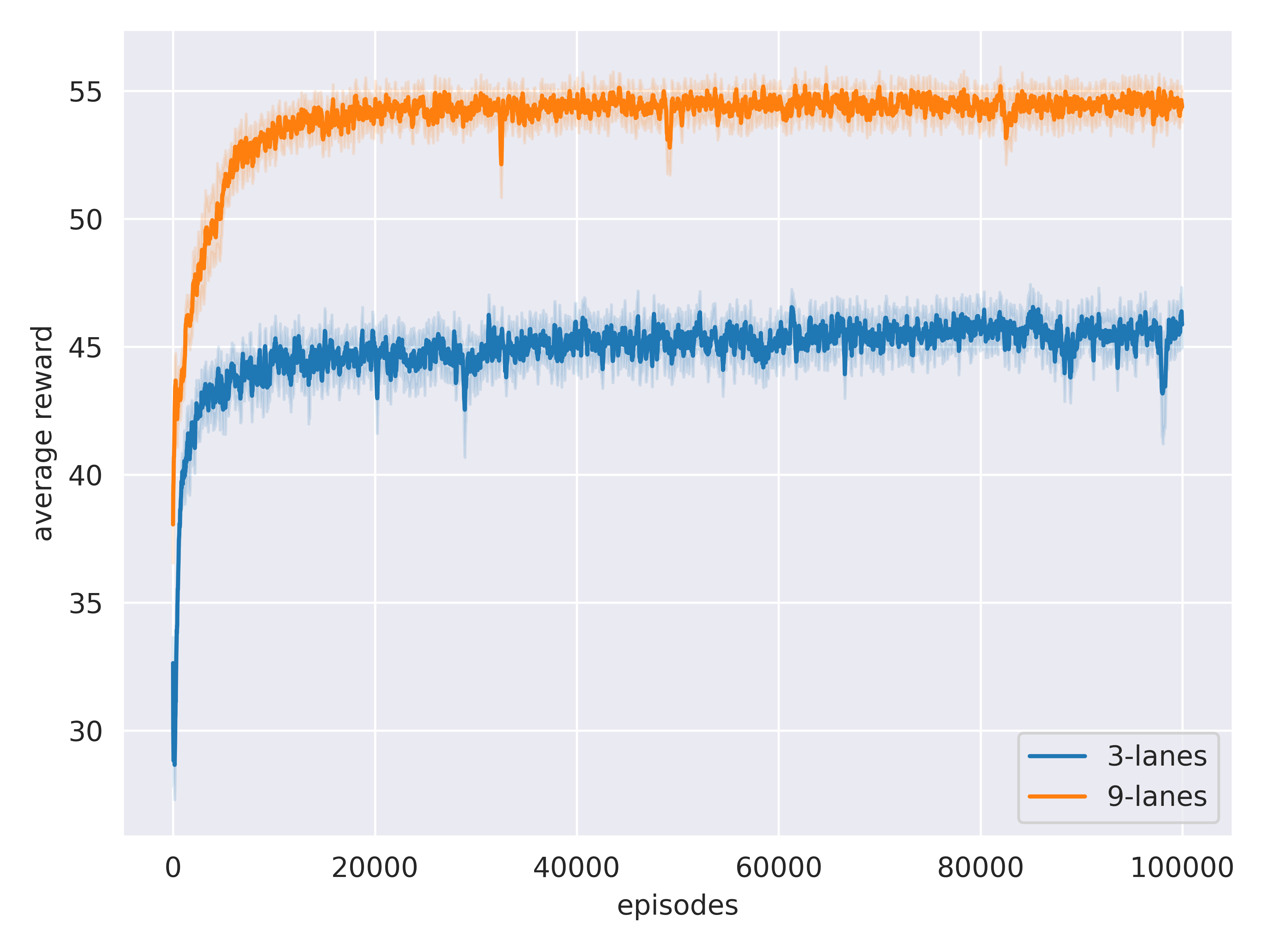}
        \caption{Return on Road-Race: 3-lane (lr=1e-3, number of particles=15) and
        9-lane (lr=3e-4, number of particles=15).}
        \label{fig:road}
    \end{subfigure}
    \begin{subfigure}[b]{.45\linewidth}
        \centering
        \includegraphics[width=1\linewidth]{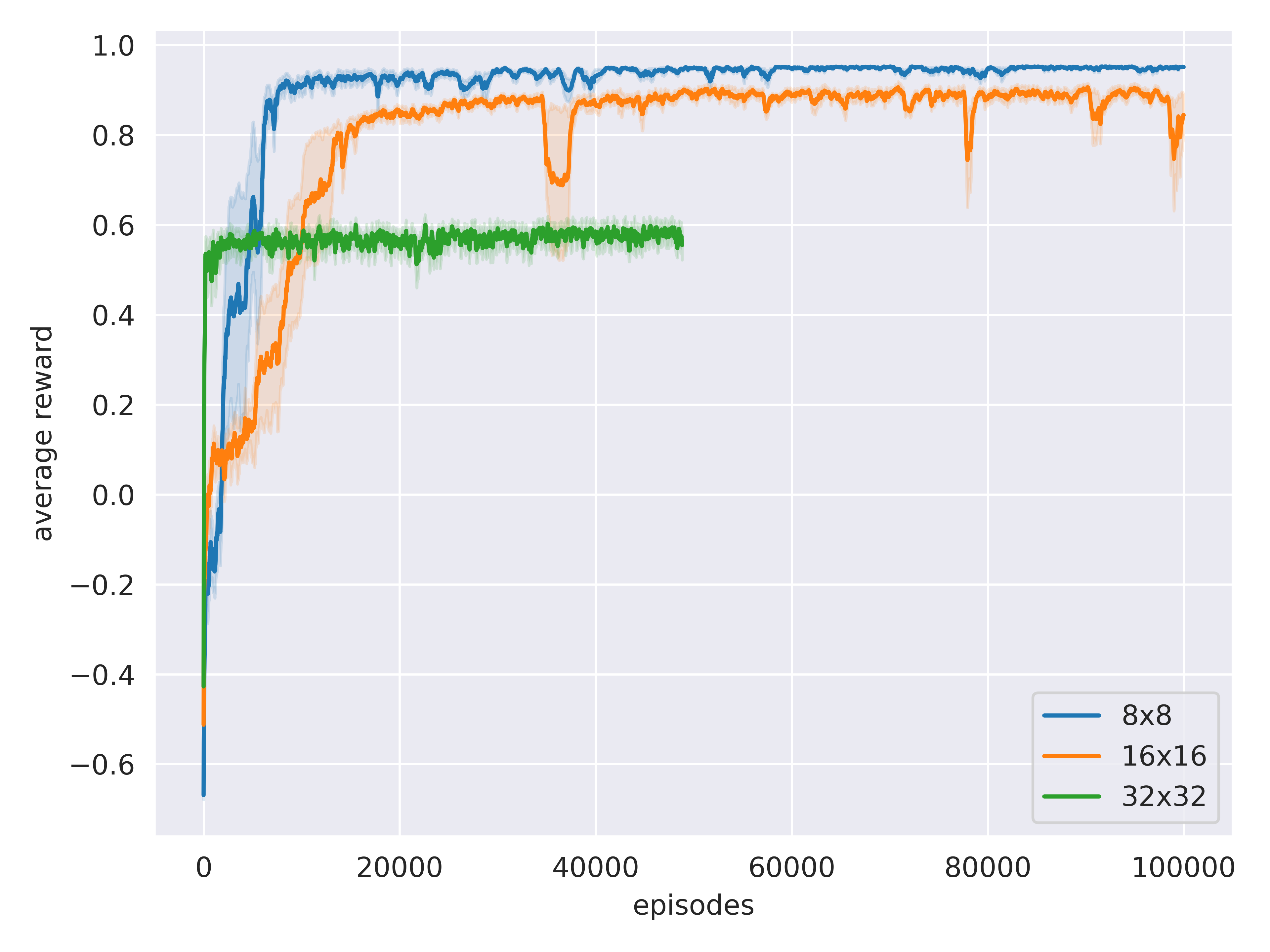}
        \caption{Return on Grid-Verse of different sizes: 8x8 (lr=1e-3, number of
        particles=30), 16x16 (lr=1e-3, number of particles=30), 32x32 (lr=1e-3,
        number of particles=15).}
        \label{fig:grid}
    \end{subfigure}
\end{figure}

In Table \ref{tab:episode-compare}, we compare the number of episodes that
DPFRL and BADDR took to achieve the same performance.

\begin{table*}[htbp]
    \centering
    \begin{tabular}{l c c c}
        \toprule
        \textbf{Domain}     & \textbf{\# Episodes for DFPRL} & \textbf{\# Episodes for BADDr} & \textbf{Return Achieved} \\
        \midrule
        Tiger               & 20k                            & 400                            & 2.2                      \\
        Collision-Avoidance & 0.6M                           & 500                            & -102                     \\
        3-lane Road-Race    & 20k                            & 200                            & 45.2                     \\
        9-lane Road-Race    & 8k                             & 200                            & 52.3                     \\
        Grid-Verse (8x8)    & 5.95k                          & 50                             & 0.8                      \\
        Grid-Verse (16x16)  & 13k                            & 50                             & 0.62                     \\
        Grid-Verse (32x32)  & N/A                            & 50                             & 0.62                     \\
        \bottomrule
    \end{tabular}
    \vspace{0.1in}
    \caption{Comparison between the number of episodes needed for DPFRL and
        BADDr to achieve the same return.}
    \label{tab:episode-compare}
\end{table*}